\documentclass[letterpaper]{article} 
\usepackage{aaai23}  
\usepackage{times}  
\usepackage{helvet}  
\usepackage{courier}  
\usepackage[hyphens]{url}  
\usepackage{graphicx} 
\usepackage{natbib}  
\usepackage{caption} 
\usepackage{algorithm}
\usepackage{algorithmic}

\usepackage{amsmath}
\usepackage{amssymb}
\usepackage{mathtools}
\usepackage{amsthm}
\usepackage{xspace}
\usepackage{graphicx}
\usepackage{subfigure}
\usepackage{wrapfig}
\usepackage{graphicx}
\usepackage{subfigure}
\usepackage{url}
\usepackage{color}
\usepackage{multirow}
\usepackage{gensymb}
\usepackage{longtable}
\usepackage{tikz}
\usepackage{comment}
\usepackage{dsfont}
\usepackage{framed}
\usepackage{enumitem}
\usepackage{xspace}
\usepackage{bm}

\usepackage{multibib}

\theoremstyle{plain}
\newtheorem{theorem}{Theorem}[section]
\newtheorem{proposition}[theorem]{Proposition}
\newtheorem{lemma}[theorem]{Lemma}

\theoremstyle{definition}
\newtheorem{definition}{Definition}

\theoremstyle{remark}
\newtheorem{remark}{Remark}

\newcommand{\ouralg}{{\texttt{LAAU}}\xspace}

\newcommand{\pone}{{{OOBC}}\xspace}
\newcommand{\avg}{{{AVG-LT}}\xspace}
\newcommand{\gradient}{{{DGD}}\xspace}
\newcommand{\mirror}{{{MW}}\xspace}
\newcommand{\rl}{{{RL}}\xspace}

\usepackage{newfloat}
\usepackage{listings}
\DeclareCaptionStyle{ruled}{labelfont=normalfont,labelsep=colon,strut=off} 
\floatstyle{ruled}
\newfloat{listing}{tb}{lst}{}
\floatname{listing}{Listing}
\title{Learning-Assisted Algorithm Unrolling for Online Optimization with Budget Constraints}
\author{
    Jianyi Yang and
    Shaolei Ren
}
\affiliations{
    University of California, Riverside\\
    \{jyang239, shaolei\}@ucr.edu

}

\usepackage{bibentry}

\begin{document}

\maketitle

\begin{abstract}
	Online optimization with multiple budget constraints 
   is challenging since the online decisions over a short time horizon  are coupled together by strict inventory constraints. The existing manually-designed algorithms cannot achieve satisfactory average performance for 
	this setting because they often need a large number of time steps
	for convergence and/or may violate the inventory constraints. 
 In this paper, we propose a new machine learning (ML) assisted
	unrolling approach, called \ouralg (Learning-Assisted
	Algorithm Unrolling), which
	unrolls the agent's online decision pipeline
	and
	leverages an ML model for updating the Lagrangian multiplier
	online. 
	For efficient training via backpropagation, we derive gradients
	of the decision pipeline over time.
	We also provide the average cost bounds for two cases when training
	data is available offline and collected online, respectively. Finally,
	 we present numerical results to highlight that \ouralg can 
	 outperform the existing
	baselines.
\end{abstract}

\section{Introduction}

Online optimization with budget (or inventory) constraints,
also referred to as \pone,
 is an important problem modeling a wide range
of  
sequential decision-making applications with limited resources,
such as online virtual machine resource allocation \cite{MungChiang_MultiResource_Fairness_ToN_Joe-Wong:2013:MAF:2578970.2578979,PalomarChiang_Distributed}, 
one-way trading  in economics \cite{one_way_trading_el2001optimal},
resource management in wireless networks \cite{Neely_Booklet,MungChiang_Fairness_Infocom_2010_lan2010axiomatic},
and
data center server provisioning \cite{resource_fairness_ghodsi2011dominant}. 
More specifically, when virtualizing a physical server into
 a small number of 
virtual machines (VMs) to satisfy
the demand of multiple sequentially-arriving jobs, the agent must
make sure that the total VM resource consumption is no more
than what the physical server can provide \cite{VMware_DPM_Whitepaper};
additional examples can be found in Appendix~\ref{sec:example}.

In an \pone problem, online actions are selected sequentially to 
 maximize the total utility
over a {short} time horizon  while 
the resource consumption  over the time horizon is 
strictly constrained by a fixed amount of budgets (i.e., violating
the budget constraint is naturally prohibited due to physical constraints). 
Consequently, the \emph{short} time horizon (e.g., 24 hourly
decisions in a day) and the \emph{strict} budget
constraint present substantial algorithmic challenges --- 
the optimal solution relies on complete offline context information, but in the online setting, only the online contexts are revealed and the exact future contexts are unavailable for decision making \cite{CR_persuit_lin2019competitive,OOIC-M_lin2022competitive}.

A relevant but different problem is online optimization with (long-term) constraints \cite{OCO_ODMD_balseiro2020dual,OCO_feldman2010online,Neely_Booklet,online_convex_paking_covering_azar2016online}. 
In the literature, a common approach  is to relax
the long-term capacity constraints and include them as
additional weighted costs into the  original optimization
objective, i.e., Lagrangian relaxation \cite{devanur2019near,OCO_zinkevich2003online,OCO_ODMD_balseiro2020dual,Neely_Booklet}.
 The Lagrangian multiplier can be interpreted as the resource \emph{price} \cite{PalomarChiang_Distributed}, and is updated at each time step by
 a manually-designed algorithm such as Dual Mirror Descent (DMD) \cite{OCO_ODMD_balseiro2020dual,online_DMD_wei2020online,OCO_jiang2020online}. 
 These algorithms require a sufficiently long time horizon for convergence, which hence may not provide satisfactory performance for short-term budget constraints, especially when contexts in an episode are not identically independently distributed (i.i.d.).
 Additionally, some studies consider constraints on average
 (i.e., equivalently, long-term constraints) \cite{Lyapunov_optimization_energy_harvesting_wireless_sensor_qiu2018lyapunov,Online_cost_aware_buffer_management_video_streaming_he2013cbm} or bound the violation of the constraints \cite{Neely_Booklet, Neely_LowComplexity_Lyapunov_JMLR_2020_JMLR:v21:16-494,online_incentive_mechanism_demand_response_colocation_datacenter_sun2016online}.
 Thus, they do not apply to \emph{strict} budget constraints
 over a short time horizon.

The challenges of \pone with short-term and strict budget constraints can be further highlighted by 
that competitive online algorithms have only been proposed very recently under settings with linear constraints
\cite{CR_persuit_lin2019competitive,OOIC-M_lin2022competitive}.
Concretely, CR-Pursuit algorithms are proposed
to make actions by following a pseudo-optimal algorithm based
on the competitive ratio pursuit framework.
Nonetheless, to make sure the solution exists for each \pone
 episode, the guaranteed competitive ratio (ratio between the algorithm cost and the offline-optimal cost) can be very large. Also, they treat each \pone problem instance as a
completely new one and focus on the worst-case competitive ratio without
considering the available historical data obtained when solving
previous \pone episodes. Thus, their conservative nature
does not result in a satisfactory average performance, which limits the practicability of these algorithms.

By tapping into the power of historical data, a natural idea for \pone
is to train an machine learning (ML) based optimizer.
Indeed, reinforcement learning has been proposed
to solve online allocation problems  in other contexts
\cite{L2O_NewDog_OldTrick_Google_ICLR_2019,L2O_OnlineBipartiteMatching_Toronto_ArXiv_2021_DBLP:journals/corr/abs-2109-10380,machine_learning_cloud_du2019learning}. But,
the existing ML-based algorithms for online optimization typically learn
 online actions in an end-to-end manner without exploiting
 the structure of the online problem being studied, which
 hence can have an unnecessarily high learning complexity
and create additional challenges for generalization to unseen problem instances
 \cite{L2O_chen2021learning,Deep_unrolling_liu2019deep}.

\textbf{Contribution.} We study
\pone with short-term and strict budget constraints,
and propose a novel ML-assisted
unrolling approach based on recurrent architectures, called \ouralg (Learning-Assisted
Algorithm Unrolling). 
Instead of using an end-to-end ML model to directly learn online actions, \ouralg uniquely exploits the \ouralg problem structure
and unrolls the agent's online decision pipeline into decision pipeline with
three stages/layers --- update the Lagrangian multiplier,
optimize decisions subject to constraints, and update remaining resource budgets
--- and only plugs an ML model into the first stage (i.e., update the Lagrangian multiplier)
where the key bottleneck for better performance exists. Thus, compared with the end-to-end model, \ouralg benefits generalization by exploiting the knowledge of decision pipeline \cite{L2O_chen2021learning}. Moreover, 
when the action dimension is larger than the number of constraints
(i.e., the dimension of Lagrangian multipliers),
the complexity advantage of using \ouralg 
to learn Lagrangian multipliers can be further enhanced compared to learning
the actions using an end-to-end model.

It is  challenging to train \ouralg through backpropagation
since the constrained optimization layer is not easily differentiable. 
Thus, we derive tractable gradients  for back-propagation through
 the optimization layer based on Karush-Kuhn-Tucker (KKT) conditions.
 In addition, 
we rigorously analyze the performance of \ouralg in
terms of the expected cost for both the case when the offline distribution information is available and the case when the data is collected online.
Finally, to validate 
\ouralg,
 we present numerical results
by considering  online resource
allocation for maximizing the weighted fairness metric.
Our results highlight that \ouralg can significantly outperform the existing
baselines
and is very close to the optimal oracle in terms
of the fairness utility.

\section{Related Works}\label{sec:related_work}

\textbf{Constrained online optimization.}
Some earlier works \cite{OCO_devanur2009adwords,OCO_feldman2010online} solve  online optimization with (long-term) constraints by estimating a fixed Lagrangian multiplier using offline data. This approach  
works only for long-term or average constraints.
 Many other studies design online algorithms
by updating the Lagrangian multiplier in an online style \cite{devanur2019near,OCO_ODMD_balseiro2020dual,online_DMD_wei2020online,OCO_jiang2020online}.  
 These algorithms guarantee sub-linear regrets under the i.i.d. context setting, and thus can achieve high utility if the number of time steps is sufficiently large.
Likewise, the Lyapunov optimization approach addresses the long-term
packing constraints by introducing virtual queues (equivalent to the Lagrangian multiplier) 
\cite{Neely_Booklet,Neely_LowComplexity_Lyapunov_JMLR_2020_JMLR:v21:16-494,Longbo_PowerOnlineLearningLyapuvno_Sigmetrics_2014_10.1145/2637364.2591990}.
Nonetheless, it also requires a sufficiently large number of time steps
for convergence.
By contrast, we consider online optimization with short-term strict budget constraints, which, motivated
by practical applications, makes \pone significantly more challenging.

Our work is relevant to the studies on \pone \cite{CR_persuit_lin2019competitive,OOIC-M_lin2022competitive} which design 
online algorithms to achieve a worst-case performance guarantee.
However, to guarantee the worst-case performance and the feasibility of the algorithm, the algorithms are very conservative and their average performances are unsatisfactory.
 Comparably, we consider a more general setting where the budget constraint can be nonlinear, and utilize available historical data more efficiently 
 to design ML-based \ouralg that unrolls the online decision pipeline and
 achieves favorable average performance.

\textbf{Algorithm unrolling.}
\ouralg is  related to the recent studies
on ML-assisted
 algorithm unrolling and deep implicit layers, which integrate
 ML into traditional algorithmic frameworks for better generalization and interpretability, lower sampling complexity and/or smaller ML model size
\cite{adler2018learned,L2O_chen2021learning,DNN_ImplicitLayers_Zico_Website,unrolling_monga2021algorithm,Deep_unrolling_liu2019deep}.
Algorithm unrolling has been used for sparse coding
 \cite{unrolling_gregor2010learning,sprechmann2015learning},
 signal and image processing \cite{unrolling_monga2021algorithm,li2019algorithm},
and solving inverse problems \cite{kobler2020total} and
 ordinary differential equations (ODEs)  \cite{chen2018neural}. Also, algorithm unrolling is applied in  \emph{learning to optimize} (L2O) \cite{L2O_chen2021learning,L2O_generalize_wichrowska2017learned}. 
 Among these works, 
 \cite{approximate_constrained_learning_Lagrandian_prediction_narasimhan2020approximate} predicts the Lagrangian multiplier by a model to efficiently solve offline optimizations which may have a large number of constraints but allow constraint violations.
These studies have their own challenges orthogonal 
to our problem where the key challenge
is the lack of complete offline information.
Thus, \ouralg, to our knowledge, is the first to leverage ML
to unroll an online optimizer for solving the online convex optimization
with budget constraints, thus having better generalization
than generic RL-based optimizers to directly obtain end solutions \cite{L2O_OnlineBipartiteMatching_Toronto_ArXiv_2021_DBLP:journals/corr/abs-2109-10380,machine_learning_cloud_du2019learning,L2O_NewDog_OldTrick_Google_ICLR_2019}.

\section{Problem Formulation}\label{sec:formulation}

As in the existing ML-based optimizers
for online problems \cite{L2O_NewDog_OldTrick_Google_ICLR_2019,L2O_OnlineBipartiteMatching_Toronto_ArXiv_2021_DBLP:journals/corr/abs-2109-10380,machine_learning_cloud_du2019learning},
 we consider an agent that interacts with
a stochastic environment.
The time horizon of an episode 
consists of $N$ time steps.
For an episode, two vectors $\bm{c}=[c_{1},\cdots, c_{N}]^\top$
and $\bm{B}=[B_1,\cdots, B_M]^\top$,  where $c_{t} $ is a context vector  and $B_m\in\mathbb{R}^+$ is the total budget for resource $m$,
 are drawn from a certain joint distribution $(\bm{c},\bm{B})\sim\mathcal{P}$, which we refer
to as the environment distribution. 
Note that $c_{i}$ and $c_{j}$ for $i\not=j$ can follow
different probability distributions, and so can $B_{i}$ and $B_{j}$ for $i\not=j$.
The random vector
$\bm{B}=[B_1,\cdots, B_M]^\top$ are revealed  at the beginning of an episode, and represents
the budgets for $M$ types of resources.
On the other hand,
$\bm{c}=[c_{1},\cdots, c_{N}]^\top$ are online contexts sequentially
revealed 
over $N$ different steps
within an episode.
That is, at step $t$, the agent only knows
$c_{1},\cdots, c_{t}$, but \emph{not} the future parameters
$c_{t+1},\cdots, c_{N}$.

At each step $t=1,\cdots,N$, the agent makes
a decision  $x_{t}\in\mathbb{R}^d$,
 consumes some budgets, and also receives
a utility. Given the decision $x_{t}$
and parameter $c_{t}$, the amount of the resource consumption
is denoted as a non-negative function $g_m(x_{t},c_{t})\geq0$, for $m=1,\cdots, M$. To be consistent with the notation of loss function,
we use a \emph{cost} or \emph{loss} $l(x_{t},c_{t})$
to denote the \emph{negative} of the utility --- the less
$l(x_{t},c_{t})$, the better. As the cost
function $l(\cdot,c_{t})$ is parameterized by $c_{t}$,
knowing $c_{t}$ is also equivalent to knowing the cost function. We assume that the loss function $l$ and the constraint functions $g_m, m=1,\cdots, M$ are twice continuously differentiable, and either the loss function $l$ or one of the constraint functions $g_m, m=1,\cdots, M$ is strongly convex in terms of the decision $x_t$.

For each episode with $(\bm{c},\bm{B})\sim\mathcal{P}$, the goal
of the agent is minimizing its total cost over the $N$ steps subject to
$M$ resource capacity constraints, which we formulate as follows:
\begin{equation}\label{eqn:target}
	\begin{split}
	&\min_{\bm{x}=(x_{1},\cdots, x_{N})}\sum_{t=1}^Nl(x_{t},c_{t}),\\
	&\text{ s.t. }\; \sum_{t=1}^N g_m(x_{t},c_{t})\leq B_m, m=1,\cdots, M.
\end{split}
\end{equation}
This is an online optimization problem
with inventory constraints (referred to as \pone)
in the sense that the short-term strict inventory constraints are imposed for each episode of
$N$ time steps. 
An episode
has its own $M$ capacity constraints which should be strictly satisfied, and the unused
budgets cannot roll over to the next episode. For ease of notation, given a policy $\pi$ which maps available inputs to \emph{feasible} actions, we denote $L(\pi)=\sum_{t=1}^Nl(x_{t},c_{t})$ as the total loss for one episode and $\mathrm{E}\left[ L(\pi)\right]$ as the expected total loss over the distribution of $(\bm{c},\bm{B})\sim\mathcal{P}$.

The setting of \pone presents new technical challenges compared with existing works on constrained online optimization. Specifically in \pone, the time horizon in an episode (episode length) is finite and can be very short. In this case, there are not many steps for algorithms to converge, and bad decisions at early steps have a large impact on the overall performance. Thus, DMD \cite{OCO_ODMD_balseiro2020dual,online_DMD_wei2020online,OCO_jiang2020online} and Lyapunov optimization \cite{Neely_Booklet} which are specifically designed for long episodes may not provide good results for \pone. Besides, unlike some studies that satisfy average constraints \cite{Lyapunov_optimization_energy_harvesting_wireless_sensor_qiu2018lyapunov,Online_cost_aware_buffer_management_video_streaming_he2013cbm} or  that only
approximately satisfy the constraints under bounded violations (i.e., \emph{soft} constraints) \cite{Neely_LowComplexity_Lyapunov_JMLR_2020_JMLR:v21:16-494,learning_aided_energy_harvesting_outdated_state_infomation_yu2019learning,online_incentive_mechanism_demand_response_colocation_datacenter_sun2016online}, \pone requires all the constraints in Eqn.~\eqref{eqn:target} be strictly satisfied. This requirement is necessary for many practical applications with finite available resources (e.g., a data center's power capacity must not be exceeded \cite{DataCenter_PowerProvisioning_Google_ISCA_2007_Fan:2007:PPW:1250662.1250665}), but makes the problem more challenging. Last but not least, the contexts in one episode in \pone are drawn from a general joint distribution (not necessarily i.i.d.). Under non-i.i.d. cases, DMD\cite{OCO_ODMD_balseiro2020dual} has performance guarantees only when each episode is long enough. CR-Pursuit\cite{CR_persuit_lin2019competitive,OOIC-M_lin2022competitive} has competitive ratios but is too conservative and may not perform well on average. To improve the average performance of \pone, new algorithms are needed to effectively utilize history data of previous episodes.

\section{Learning-Assisted Algorithm Unrolling}\label{sec:algorithm}

\subsection{Relaxed Optimization}
The design of \ouralg is based on the Lagrangian relaxed optimization method which is introduced here.
 Since it is difficult to directly solve the constrained optimization in Eqn.~\eqref{eqn:target} due to the lack of complete offline information in an online setting, many studies \cite{OCO_MULPLICATIVE_arora2012multiplicative,OCO_ODMD_balseiro2020dual,Neely_Booklet} solve the Lagrangian relaxed form  
written as follows:
\begin{equation}\label{eqn:lagrangian_general}
	\begin{split}
		\min_{\bm{x}=(x_{1},\cdots, x_{N})}\sum_{t=1}^Nl(x_{t},c_{t})
+ \lambda^\top \sum_{t=1}^N\bm{g}(x_{t},c_{t}),
\end{split}
\end{equation}
where $\bm{g}(x_{t},c_{t})=[g_1(x_{t},c_t),\cdots, g_M(x_{t},c_t)]^\top$
and $\lambda=[\lambda_{1},\cdots, \lambda_{M}]^\top$ is
the non-negative Lagrangian multiplier corresponding to the $M$ constraints $\sum_{t=1}^N\bm{g}(x_{t},c_{t})\leq \bm{B}$.
The multiplier $\lambda$ with $M$ dimensions essentially relaxes the $M$ inventory constraints,
thus decoupling the decisions over the $N$ time steps within an episode.
It is also interpreted as the resource \emph{price} in the resource allocation literature \cite{PalomarChiang_Distributed,boyd2004convex,Neely_Booklet}:  A greater
$\lambda_{t}$ means a higher price for the resource consumption, thus pushing
the agent to use less resource.
Clearly, had we known the optimal Lagrangian multiplier $\lambda^*$
at the beginning of each episode, the \pone problem would become very easy.
Unfortunately, 
knowing $\lambda^*$ also requires the complete offline information $(\bm{c},\bm{B})$,
which is not possible in the online case.
 Nevertheless, 
if we can appropriately update $\lambda_{t}$
in an online manner while strictly satisfying the constraints, we can also efficiently solve the \pone problem.   
Formally, by using $\lambda_{t}$ that is updated online
 for each step $t$,
we can instead solve the following relaxed problem:
\begin{equation}\label{eqn:balancedtarget}
	\begin{split}
		\min_{x_t} l(x_{t},c_{t})+\lambda_{t}^\top \bm{g}(x_{t},c_{t}),
		\;\text{ s.t., }\; \bm{g}(x_{t},c_{t})\leq \bm{b}_t,
	\end{split}
\end{equation}
where $\lambda_{t}=[\lambda_{t,1},\cdots, \lambda_{t,M}]^\top$, $\bm{g}(x_{t},c_t)=[g_1(x_{t},c_t),\cdots, g_M(x_{t},c_t)]^\top$, and the remaining budget for step $t$ is $\bm{b}_t=\bm{B}-\sum_{s=1}^{t-1}\bm{g}(x_{s},c_{s})$.

In fact, designing good update rules for $\lambda_{t}$ for each step $t=1,\cdots,N$
is commonly considered
in the literature
\cite{OCO_ODMD_balseiro2020dual,OCO_agrawal2014fast}.
 For example, \cite{OCO_ODMD_balseiro2020dual} update $\lambda_{t}$ by DMD, in order to meet \emph{long}-term constraints while achieving a
low regret compared to the optimal oracle. Nonetheless,
it requires a large number of time steps to converge to a good Lagrangian parameter.
Likewise,
the Lyapunov optimization technique  
introduces a virtual queue, whose length
essentially takes the role of $\lambda_{t}$
and is updated
as $\lambda_{t+1}=\max\left\{\lambda_{t}+ \bm{g}(x_{t},c_{t})
- \frac{1}{N}\bm{B},0\right\}$ or in other similar ways
\cite{Neely_Booklet}. Nonetheless,
the convergence rate of using Lyapunov optimization
is slow (even assuming $c_{t}$ is i.i.d. for $t=1,\cdots,N$),
and there exists a tradeoff
between
cost minimization and long-term constraint satisfactory, making it unsuitable for the short-term constraints
that we focus on.

Alternatively, one may want to exploit
the distribution information of $(\bm{c},\bm{B})\sim\mathcal{P}$
and solve a relaxed problem offline by considering $M$  average
constraints (referred to as \avg). That is,
we replace the
short-term capacity constraints in Eqn.~\eqref{eqn:target}
with  $\mathrm{E}_{\mathcal{P}}\left[\sum_{t=1}^N g_m(x_{t},c_{t})\right]\leq B_m$
for $m=1,\cdots,M$. By solving this relaxed problem, we can
 obtain
 a  Lagrangian multiplier
$\lambda_{\mathcal{P}}$ that only depends on
$\mathcal{P}$ but not the specific $(\bm{c},\bm{B})$. Thus, we can
replace $\lambda_t$
in Eqn.~\eqref{eqn:balancedtarget} with $\lambda_{\mathcal{P}}$. 
However, since this method uses a constant Lagrangian multiplier for all episodes, 
we will either be overly conservative
and not using the budgets as much as possible, 
or violating the
the  constraints.
\subsection{Algorithm Unrolling}

\begin{algorithm}[!t]
	\caption{Online Inference Procedure of \ouralg}
	\begin{algorithmic}[1]\label{alg:online_inference}
		\REQUIRE ML model $f_{\theta}$.
		\FOR   {t=1 to N}
		\STATE Receive $c_t$, forward propagate $f_{\theta}$ and get Lagrangian multiplier $\lambda_t=f_{\theta}\left(\bm{b}_{t},c_{t},\bar{t} \right)$.
		\STATE Solve the constrained convex optimization in \eqref{eqn:optlayer} and make action $x_t$.
		\STATE Update the resource budget $\bm{b}_{t+1}=\bm{b}_{t}-\bm{g}(x_t,c_{t})$.\label{alg:updating}
	\ENDFOR
	\end{algorithmic}
\end{algorithm}

We propose to leverage the powerful capacity of ML to find a solution.  One approach
is to train an end-to-end model that takes the online input
information and directly outputs a decision.
 But, the end-to-end model should be large enough to capture the possibly complex logic of the optimal policy, and the end-to-end models often have
 poor interpretability and worse
 generalization (see the
 comparison between \ouralg and
 the generic end-to-end approach in Section~\ref{sec:simulation}). 
Therefore, 
instead of replacing the whole decision pipeline with ML, 
\emph{we only
plug an ML model in the most challenging stage} --- online updating of the
Lagrangian multiplier
$\lambda_{t}$  needed to solve the relaxed problem
in Eqn.~\eqref{eqn:balancedtarget}.

\begin{figure}
	\includegraphics[width={0.45\textwidth}]{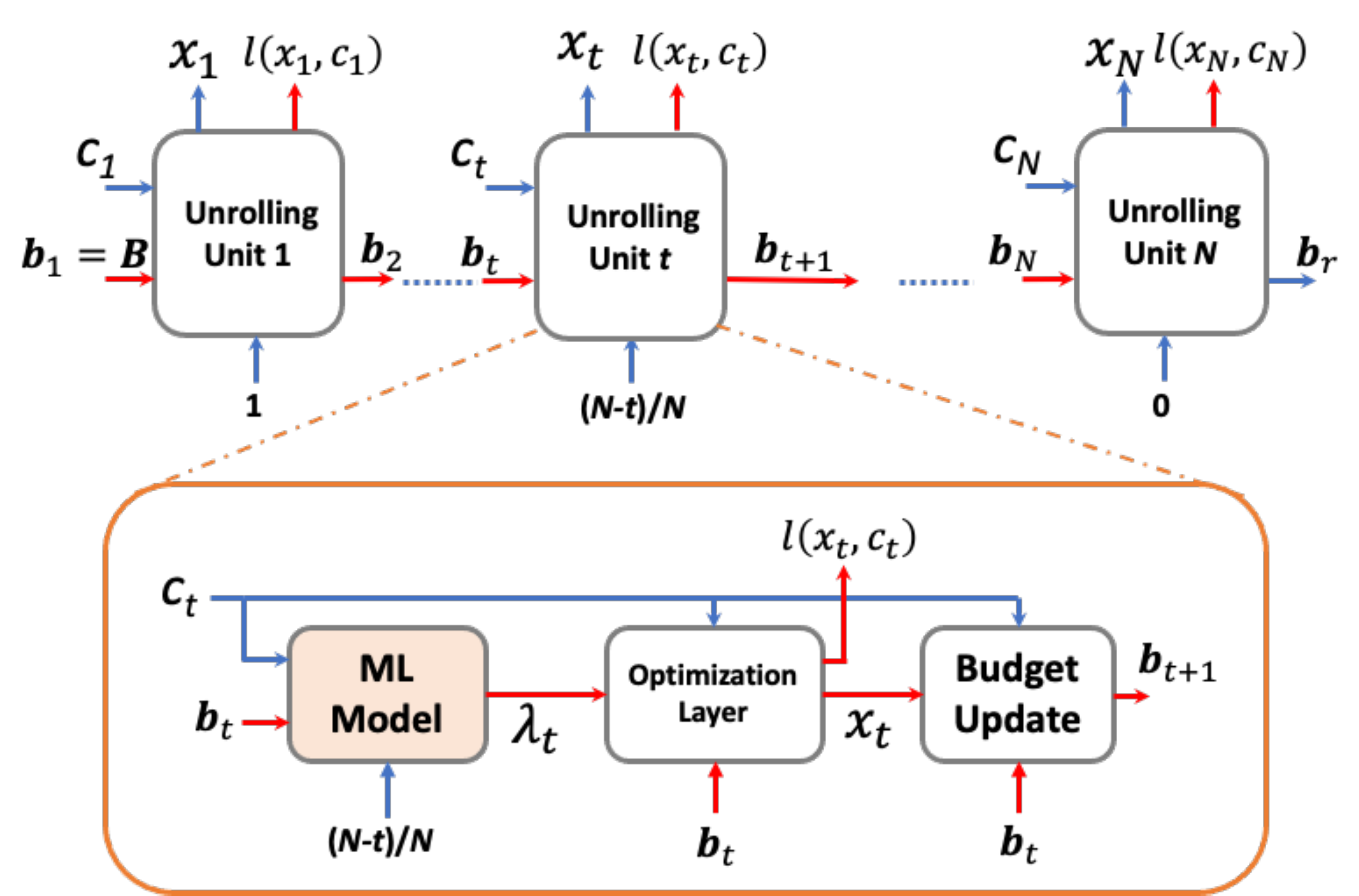}
	\caption{Architecture of \ouralg.
		The red lines indicate the flows that need back propagation.}
  \vspace{-0.3cm}	
\label{fig:illustration}
\end{figure}
As shown in Algorithm \ref{alg:online_inference} and illustrated in Fig.~\ref{fig:illustration},
the decision pipeline at step $t$
can be decomposed into three stages as follows.

\textbf{Updating $\lambda_{t}$.}
At the beginning of step $t=1,\cdots,N$,
the ML model takes
 the parameter $c_{t}$, the remaining budget $\bm{b}_{t}=[b_{t,1},\cdots, b_{t,m}]^\top$ and the normalized number of  remaining steps $\bar{t}=\frac{N-t}{N}$ as the inputs,  and outputs the Lagrangian multiplier $\lambda_t=[\lambda_{t,1},\cdots, \lambda_{t,m}]^\top$.
 Letting $f_{\theta}$ denote the ML model parameterized by  $\theta$,  we have
$
	\lambda_t=f_{\theta}\left(\bm{b}_{t},c_{t},\bar{t} \right).
$

\textbf{Optimization layer.}
In the
 optimization layer, we solve a relaxed convex problem
formulated in Eqn.~\eqref{eqn:balancedtarget}. The natural constraints
on the remaining resource budgets ensure that the \emph{strict} inventory constraints are always satisfied by \ouralg.
We denote the optimization layer as
$p(c_{t},\lambda_{t},\bm{b}_t)$ and thus have:
\begin{equation}\label{eqn:optlayer}
	\begin{split}
	x_t=p(c_{t},\lambda_{t},\bm{b}_t)\!=&\!\arg\!\min_{x} \!\left( l(x,c_{t})+\lambda_{t}^\top \bm{g}(x,c_{t})\right),\\
	 &\text{ s.t., } \bm{g}(x,c_{t})\leq \bm{b}_t.
	\end{split}
	\end{equation}

\textbf{Updating resource budgets.}
In the last stage, the remaining resource budgets 
serve as an input for the next recurrence and
are updated as
$
\bm{b}_{t+1}=\bm{B}-\sum_{s=1}^{t}\bm{g}(x_s,c_{s})=\bm{b}_{t}-\bm{g}(x_t,c_{t})
$.

Each episode
 includes $N$ recurrences, each for one decision step.
The cost for step $t$ is calculated after the optimization layer as
$l(x_t,c_t)$. 
Note that in the $N$-th  recurrence which is the final stopping step,  the remaining budget $\bm{b}_r=\bm{b}_N-\bm{g}(x_N,c_{N})$ will be wasted if not used up.  
Thus, we directly set $\lambda_N=0$  for the $N$-th unrolling unit.

\section{Training the Unrolling Architecture}\label{sec:training}
To clearly explain the back propagation of the unrolling model, we first consider the offline training where offline distribution information is available. Then we extend to the online training setting where the data is collected online.
\subsection{Offline Training}
\subsubsection{Training Objective}
For the ease of notation,
we denote the online optimizer as $\bm{h}_\theta\left(\bm{B},\bm{c} \right)$.
For offline training, we are given an \emph{unlabeled} training dataset $S=\left\lbrace \left( \bm{c}_1,\bm{B}_1\right) ,\cdots, \left( \bm{c}_n,\bm{B}_n\right) \right\rbrace $,
 with $n$ samples of $(\bm{c}, \bm{B})$.
The training dataset can be
 synthetically generated
by sampling from the target distribution for the online input
 $(\bm{c}, \bm{B})$,
 which is a standard technique in the context of \emph{learning to optimize} \cite{L2O_chen2021learning,L2O_Combinatorial_Graphs_2017_co_graphs_nips_2017,L2O_LearningToOptimize_Berkeley_ICLR_2017,L2O_OnlineBipartiteMatching_Toronto_ArXiv_2021_DBLP:journals/corr/abs-2109-10380,machine_learning_cloud_du2019learning}.
 By forward propagation, we can get the empirical training loss  as $L(\bm{h}_\theta,S)=\frac{1}{n}\sum_{i=1}^n\sum_{t=1}^Nl(\bm{x}_{i,t},\bm{c}_{i,t})$ where $\bm{x}_{i,t}$ is the output of the online optimizer $\bm{h}_\theta$ regarding $\bm{c}_i$ and $\bm{B}_i$.
By minimizing the empirical loss, we get $\hat{\theta}=\arg\min_{\theta}L(\bm{h}_\theta,S)$.

\subsubsection{Backpropagation}
 Typically,
the minimization of the training loss is performed by gradient descent-based algorithms like SGD or Adam, which need back propagation to get the gradient of the loss with respect to the ML model weight $\theta$.
Nonetheless, unlike standard ML training (e.g., neural network training with only linear and activation operations), our unrolled recurrent architecture includes
an \emph{implicit} layer --- the optimization layer \cite{DNN_ImplicitLayers_Zico_Website}.
Additionally, the unrolling architecture has multiple skip connections.
Thus, the back-propagation process is dramatically different from that of  standard recurrent neural networks.
Next, we derive the gradients for back propagation in our unrolling design.
Note that the loss $l(x_t,c_t)$ for any $t=1,\cdots,N$ is directly determined by the output of the optimization layer $x_t$ and the parameter $c_t$, and $x_t$ needs back propagation. Thus, by the chain rule, we have
\begin{equation}\label{eqn:lossgradient}
	\bigtriangledown_{\theta}l\!\left(x_t,c_t \right)\!\!=\!\! 	 \bigtriangledown_{x_{t}}\!l\!\left(x_t,c_t \right)\!\!\left( \!\bigtriangledown_{\lambda_t}\!x_t \!\!\bigtriangledown_\theta \lambda_t\!\!+\!\!\bigtriangledown_{\bm{b}_t}\!x_t \!\!\bigtriangledown_\theta \bm{b}_t\right).
	\end{equation}

To get $\bigtriangledown_{\lambda_t}\!x_t $ and $\bigtriangledown_{\bm{b}_t}\!x_t $ in Eqn. \eqref{eqn:lossgradient}, we need to perform back propagation for the optimization layer $p(c_{t},\lambda_{t},\bm{b}_t)$. This is a challenging task and will be addressed in Section~\ref{sec:differentiation}.
The other gradients in Eqn.~\eqref{eqn:lossgradient} include $\bigtriangledown_\theta \lambda_t$ and $\bigtriangledown_\theta \bm{b}_t$. Note that $\lambda_t$,
which  is the ML model output directly determined by its ML model weight $\theta$,
and
the remaining budget $\bm{b}_{t}$ both need back propagation. Thus, the gradient of $\lambda_t$ with respect to the ML model weight $\theta$ is expressed as
\begin{equation}\label{eqn:gradlambda}
	\bigtriangledown_\theta \lambda_t=\bigtriangledown_\theta f_{\theta}\left(\bm{b}_{t},c_{t},\bar{t} \right)+\bigtriangledown_{\bm{b}_{t}} f_{\theta}\left(\bm{b}_{t},c_{t},\bar{t} \right)\bigtriangledown_\theta\bm{b}_{t},
\end{equation}
Now, it remains to derive $\bigtriangledown_\theta\bm{b}_{t}$, which is important since $b_t$ is the signal connecting two adjacent recurrences. By the expression of $b_t$ in Line \ref{alg:updating} of Algorithm \ref{alg:online_inference}, we have
\begin{equation}\label{eqn:gradbudget}
	\bigtriangledown_\theta \bm{b}_t=\bigtriangledown_\theta \bm{b}_{t-1}+\bigtriangledown_{x_{t-1}} \bm{g}(x_{t-1},c_{t-1})\bigtriangledown_{\lambda_{t-1}}\!x_{t-1} \bigtriangledown_\theta \lambda_{t-1},
\end{equation}
Combining Eqn.~\eqref{eqn:lossgradient}, \eqref{eqn:gradlambda} and \eqref{eqn:gradbudget}, we get the recurrent expression for back propagation.
Then, by adding up the gradients of the losses over $N$ time steps,
we get the gradient of the total loss as
$\bigtriangledown_\theta L (\bm{h}_{\theta},S)=\frac{1}{n}\sum_{i=1}^n\sum_{t=1}^N\bigtriangledown_{\theta}l\left(x_{i,t},c_{i,t} \right)$.

\subsubsection{Differentiating the Optimization Layer}\label{sec:differentiation}
It is challenging to get the close-form solution and its gradients for many constrained optimization problems.
One possible remedy
is to use some black-box gradient estimators like zero-order optimization \cite{zero-order_opt_signalprocessing_ML_liu2020primer,zero_order_optimization_tutorial_ruffio2011tutorial}. However, zero-order gradient estimators are not computationally efficient since many samples are needed to estimate a gradient. 
Another method is to train a deep neural network to approximate the optimization layer in  Eqn.~\eqref{eqn:optlayer} and then calculate the gradients based on the neural network. However, we need many samples to pre-train the neural network, and the gradient estimation error can be large.
we analytically differentiate the solution to Eqn.~\eqref{eqn:optlayer} 
in the optimization layer
with respect to the inputs $\lambda_t$, and $\bm{b}_t$ by exploiting  KKT conditions  \cite{DNN_ImplicitLayers_Zico_Website,boyd2004convex}. The KKT-based differentiation method, given in Proposition~\ref{lma:grad_optlayer}, 
is computationally efficient, explainable and accurate (under mild technical conditions).
\begin{proposition}[Back-propagation by KKT]\label{lma:grad_optlayer}
	Assume that $x_t$ and $\mu_t$ are the primal and dual solutions to Eqn.~\eqref{eqn:optlayer} , 	respectively.
	Let $\Delta_{11}=\bigtriangledown_{x_t x_t}l\left( x_t, c_t\right)+\sum_{m=1}^M\left(  \lambda_{m,t} +\mu_{m,t}\right) \!\!\bigtriangledown_{x_t x_t}g_m\!\left(x_{t},c_{t} \right)$, $\Delta_{12}=\left[ \bigtriangledown_{x_{t}}\bm{g}\left(x_{t},c_{t} \right)\right]^\top $, $\Delta_{21}=\mathrm{diag}(\mu_t)\bigtriangledown_{x_{t}}\bm{g}\left(x_{t},c_{t} \right)$, and $\Delta_{22}=\mathrm{diag}\left( \bm{g}\left(x_{t},c_{t} \right)-\bm{B}_{t}\right)$. 
	If the conditions in Proposition~\ref{lma:condition_optlayer} are satisfied, the gradients of the optimization layer  w.r.t.  $\lambda_t$ and $\bm{b}_t$ are 
	\[
	\bigtriangledown_{\lambda_t}x_{t}= -\left(\Delta_{11}^{-1}+\Delta_{11}^{-1}\Delta_{12}\mathrm{Sc}\left(\Delta, \Delta_{11}\right)^{-1}\Delta_{21}\Delta_{11}^{-1} \right) \Delta_{12},
	\]
	\[
	\bigtriangledown_{\bm{b}_t}x_{t}= -\Delta_{11}^{-1}\Delta_{12}\mathrm{Sc}\left(\Delta, \Delta_{11}\right)^{-1}\mathrm{diag}(\mu_t),
	\]
	where $\mathrm{Sc}\left(\Delta, \Delta_{11}\right)=\Delta_{22}-\Delta_{21}\Delta^{-1}_{11}\Delta_{12} $ denotes the Shur-complement of $\Delta_{11}$ in 
	$\Delta=\left[[\Delta_{11},\Delta_{12}]; [\Delta_{21},\Delta_{22}]\right] $.
	$\hfill{\square}$
\end{proposition}
We find that to get  truly accurate gradient computation by Proposition~\ref{lma:condition_optlayer},  the Shur-complement $\mathrm{Sc}\left(\Delta, \Delta_{11}\right)$ and $\Delta_{11}$ should be invertible. Otherwise, we can get approximated gradients by taking pseudo-inverse of $\mathrm{Sc}\left(\Delta, \Delta_{11}\right)$ and $\Delta_{11}$. The sufficient conditions to guarantee perfectly accurate gradient computation are given in Proposition~\ref{lma:condition_optlayer}.
\begin{proposition}[Sufficient Conditions of Accurate Differentiation]\label{lma:condition_optlayer}
	Assume that the  problem in Eqn.~\eqref{eqn:optlayer} satisfies strong duality. The loss $l$ or one of  the constraints $g_m,m=1,\cdots, M$ is strongly convex with respect to $x$.
Denote $\mathcal{A}$ as the  index set of constraints that are activated (i.e., equality holds) under the optimal solution.
	$\mu_{m,t}\neq 0, \forall m\in\mathcal{A}$, the size of the activation set satisfies $ |\mathcal{A}|\leq d$ with $d$ as the action dimension, and the gradients $\bigtriangledown_{x_t}g_m(x_t), m\in\mathcal{A}$ are linearly independent and not zero vectors, the gradients in Proposition~\ref{lma:grad_optlayer} are perfectly accurate.
\end{proposition}

\begin{remark}\label{remark1}
To derive the gradient of Eqn.~\eqref{eqn:optlayer} with respect to an input parameter, we take gradients on both sides of the equations in KKT conditions 
by the chain rule and get new equations about the gradients. 
 By solving the obtained set of equations and exploiting the block matrix inversion, we can derive the gradients with respect to the inputs in Proposition~\ref{lma:grad_optlayer}.

The conditions in Proposition \ref{lma:condition_optlayer} are mild in practice.  First, strong duality is easily satisfied for the considered convex optimization in Eqn.~\eqref{eqn:optlayer} given the Slater's condition \cite{boyd2004convex}.  Besides, the requirement of strong convexity excludes linear programming (LP). Actually, LP problems with resource constraints  are usually solved by other relaxations other than our considered relaxation in Eqn. \eqref{eqn:lagrangian_general} \cite{devanur2019near}.
The other conditions are related to activated constraints.  According to the condition of complementary slackness  \cite{boyd2004convex}, the condition that optimal dual variables corresponding to the activated constraints are not zero typically holds. We also require that the number of activated constraints is less than the action dimension, and the gradient vectors of the activated constraint functions under optimal solutions should be independent from each other. Given that at most a small number of 
constraints are activated in most cases, the two conditions are easily satisfied. Actually, the independence condition requires that the activated constraints are not redundant --- an activated constraint function is not a linear combination of any other activated constraint functions; otherwise, it can be replaced by other constraints.   The proof of  Proposition~\ref{lma:condition_optlayer} is given in the appendix \ref{sec:proofbackpropagation_section5}. 
\end{remark}
\subsection{Online Training}\label{sec:online_training}
In practice, we may have a cold-start setting without
many offline samples.
An efficient approach for this setting is  online stochastic gradient descent (SGD) with its algorithm in Appendix \ref{sec:online_training_alg}. Concretely, when the $i$-th instance arrives, we perform online inference by Algorithm~\ref{alg:online_inference}. After the instance with $N$ steps ends, we collect the context and budget data of this episode and update the ML model weight $\hat{\theta}_i$ by performing one-step gradient descent,  
i,e, $\hat{\theta}_i=\hat{\theta}_{i-1}-\bar{\alpha}\triangledown_{\theta} L(h_{\hat{\theta}_{i-1}},\bm{c}_i)$ where $L(h_{\hat{\theta}_{i-1}},\bm{c}_i)$ is the loss of the unrolling model for the $i$th instance and $\bar{\alpha}$ is the stepsize. The back propagation method is the same as the offline training. Then, with the updated
$\hat{\theta}_i$, we perform inference by Algorithm \ref{alg:online_inference} for the instance in the $(i+1)$-th round. We will show by analysis that the average cost decreases with time.

\section{Performance Analysis}\label{sec:analysis}

In this section, we bound the expected cost when the trained ML model $f_\theta$
is used in \ouralg. 

\begin{definition}\label{def:optimal_weight}
	The weight in
the ML model $f_{\theta}$
 (and also
the online optimizer $\bm{h}_{\theta}$)
that minimizes the expected loss $\mathrm{E}\left[ L(\bm{h}_{\theta},\bm{c})\right] $ with respect to the distribution of $(\bm{c},\bm{B})\sim\mathcal{P}$ is defined as
	$
	\theta^*=\arg\min_{\theta\in\Theta}\mathrm{E}\left[ L(\bm{h}_{\theta},\bm{c})\right],
	$
	and the weight that minimizes the empirical loss $L(\bm{h}_{\theta},S)$ is defined as
	$
	\hat{\theta}^*=\arg\min_{\theta\in\Theta}L(\bm{h}_{\theta},S),
	$
	where $\Theta$ is the weight space.
\end{definition}

In Definition~\ref{def:optimal_weight}, given
the weight space $\Theta$,
$\bm{h}_{\theta^*}$ is the best online optimizer based on the unrolling architecture
in terms of the expected cost.
$\bm{h}_{\theta^*}$  is not the offline-optimal policy, but it is close to the policy that performs best given available online information when the capacity of the ML model
and weight space $\Theta$
are large enough.
Next, we show the performance gap of \ouralg compared with
 $h_{\theta^*}$.

\begin{theorem}\label{thm:generalization}
By the optimization layer
in Eqn.~\eqref{eqn:optlayer}, \ouralg satisfies the inventory constraints for each \pone instance. 
	 Suppose that $\hat{\theta}$ is the ML model weight by offline training on dataset $S$ with $n$ samples, and that we plug it into the online optimizer $\bm{h}_{\hat{\theta}}$. With probability at least $1-\delta,\delta\in(0,1)$,
\begin{equation}\label{eqn:theorem_bound}
	\begin{split}
	\mathrm{E}\left[ L(\bm{h}_{\hat{\theta}})\right] \!\!-\!\!\mathrm{E}\left[ L(\bm{h}_{\theta^*})\right] 
 \leq \mathcal{E}\left(\bm{h}_{\hat{\theta}},S\right) +4R_n(L\circ \mathbb{H})\\+2\left( \Gamma_{L,c}\omega_c+\Gamma_{L,b}\omega_b\right) \sqrt{\frac{\ln(2/\delta)}{n}},
	\end{split}
\end{equation}
	where $\mathcal{E}\left(\bm{h}_{\hat{\theta}},S\right)  =L(\bm{h}_{\hat{\theta}}, S)-L(\bm{h}_{\hat{\theta}^*}, S)$ is the training error, $R_n(L\circ\mathbb{H})$ is the Rademacher complexity regarding the loss space $L\circ \mathbb{H}=\left\lbrace L(\bm{h}), \bm{h}\in\mathbb{H}\right\rbrace $  with $\mathbb{H}$ being the ML model set,
 $\omega_c=\max_{\bm{c},\bm{c}'\in \mathbb{C}}\left\|\bm{c}-\bm{c}' \right\| $ is the size of the parameter space $\mathbb{C}$, $\omega_b=\max_{\bm{B},\bm{B}'\in \mathbb{B}}\left\|\bm{B}-\bm{B}' \right\| $ is the size of the capacity constraint space $\mathbb{B}$, $\Gamma_{L,c}$ and $\Gamma_{L,b}$ 
are the Lipschitz constants of  the total loss
$L(\bm{h}_{\theta},\bm{c})=\sum_{t=1}^Nl(x_t,c_t)$ with respect to $\bm{c}$ and  $\bm{B}$, 
respectively.
\end{theorem}
\begin{proposition}\label{lma:rad_linear}
	If a linear model $f_{\theta}(\bm{v})=\theta^\top\phi(\bm{v}), \|\theta\|\leq Z$ is used as the ML model in \ouralg, the Rademacher complexity $R_n(L\circ \mathbb{H})$ is bounded ny
	$O\left(\frac{ZW}{\sqrt{n}}\right),
	$
	where $W=\sup_{\bm{v}}\sqrt{\phi(\bm{v})^\top \phi(\bm{v}})$.
	If a neural network, where the depth is $K$, the width is less than $u$,  activation functions are $\Gamma_{\alpha}$-Lipschitz continuous, and the spectrum norm of the weight matrix in layer $k$ with is less than $Z_k$, is used as the ML model,  the Rademacher complexity $R_n(L\circ \mathbb{H})$ is bounded by
	$
	R_n(L\circ \mathbb{H})
	\leq O\left(\frac{K^{3/2}u\Gamma_{\alpha}(\beta_b+\beta_c)\prod_{k=1}^{K}Z_k}{\sqrt{n}} \right),
	$
	where $\beta_b, \beta_c$ are the largest $l_2$-norm of $\bm{B}$ and $\bm{c}$. 	The notation $O$ in this proposition indicates the scaling relying on $M$, $N$, the Lipschitz constants of loss function $l$, constraint function $g$, optimization layer $p$ and neural network $f$.
\end{proposition}

\begin{remark}
	Theorem~\ref{thm:generalization} shows that the performance gap between \ouralg and the pseudo-oracle $h_{\theta^*}$ in terms of expected loss is bounded by the empirical training error,
	plus a generalization error which relies on the Radmacher complexity, the LipSchitz constant of the online optimizer, and the number of training samples. The Radmacher complexity indicates the richness of the loss function space with respect to the online optimizer space $\mathbb{H}$ and the distribution $\mathcal{P}$ and is further bounded in Proposition~\ref{lma:rad_linear}. From the bound of Radmacher complexity,  we find that for both linear model and neural network, the generalization error increases with episode length $N$ and the number of constraints $M$. Besides, the Rademacher complexity relies on the ML model designs. For example, if a linear model is used as the ML model, the generalization error relies on the norm bounds of feature mapping  and linear weights, while if a neural network is used as the ML model, the generalization error is related to  the network length , width, the smoothness of activation functions, and spectral norm bounds of the weights in each layer.
	The last term of the expected cost bound is scaled by $\left( \Gamma_{L,c}\omega_c+\Gamma_{L,b}\omega_b\right)$ which indicates the sensitivity of the loss when the inputs are changed.
	This term highly depends on the Lipschitz constants of the total loss regarding the two inputs. 
	Our proof in the appendix gives the bound of 
	the Lipschitz constants of the loss
	regarding its inputs. \hfill{$\square$}
\end{remark}

\begin{proposition}[Average Cost of Online Training]\label{thm:generalization_sgd}
Assume that for each round $i$, $\triangledown_{\theta} \mathrm{E}[L(h_{\hat{\theta}_i})]$ is $\Gamma_{\triangledown L,\theta}$-Lipschitz continuous, and  the Polyak-Lojasiewicz inequality is satisfied, i.e. $\exists \varsigma>0, \left\|\triangledown_{\theta} \mathrm{E}[L(h_{\hat{\theta}_i})]\right\|^2\geq 2\varsigma \left(\mathrm{E}[L(h_{\hat{\theta}_i})]-\mathrm{E}[L(h_{\theta^*})] \right)$. Also, assume that for the distribution of $\bm{c}_i, $, $\exists \iota_G>\iota >0$ such that $\forall i$, $\left\langle\triangledown_{\theta} \mathrm{E}[L(h_{\hat{\theta}_i})],\mathrm{E}\left[\triangledown_{\theta} L(h_{\hat{\theta}_i},\bm{c}_i)\right]\right\rangle\geq \iota\left\| \triangledown_{\theta} \mathrm{E}[L(h_{\hat{\theta}_i})] \right\|^2_2$, $\|\mathrm{E}[\triangledown_{\theta} L(h_{\hat{\theta}_i},\bm{c}_i)] \|_2\leq \iota_G \left\| \triangledown_{\theta} \mathrm{E}[L(h_{\hat{\theta}_i})] \right\|_2$,  and there exist $  \varpi, \varpi_V>0$ such that $\forall i$, $\mathrm{Var}\left[\triangledown_{\theta} L(h_{\hat{\theta}_i},\bm{c}_i)\right]\leq \varpi+\varpi_V\left\| \triangledown_{\theta} \mathrm{E}[L(h_{\hat{\theta}_i})] \right\|^2_2$. Then with the same notations as Theorem \ref{thm:generalization}, for the online setting where \ouralg is trained by SGD with stepsize $0< \bar{\alpha}\leq \frac{\iota}{\Gamma_{\triangledown L,\theta}(\varpi_V+\iota_G^2)}$, with probability at least $1-\delta,\delta\in(0,1)$,
we have for each round $i$
\begin{equation}\label{eqn:theorem_bound_sgd}
	\begin{split}
	\!\mathrm{E}\!\left[\! L(\bm{h}_{\hat{\theta_i}})\!-\! L(\bm{h}_{\theta^*})\right]\! 
 \leq \!\frac{\bar{\alpha}\Gamma_{\triangledown L, \theta}\varpi}{2\iota\varsigma}+ O\left(\left(1-\bar{\alpha}\iota\varsigma\right)^i\right),
	\end{split}
\end{equation}
where the expectation is taken over the randomness of context $\bm{c}$ and budget $\bm{B}$ and the model weight $\hat{\theta}_i$ by SGD.
\end{proposition}
Proposition~\ref{thm:generalization_sgd}, with proof in appendix \ref{sec:generalization_proof}, bounds the expected loss gap between the learned weight $\hat{\theta}_i$ by SGD and the optimal weight $\theta^*$ defined in Definition \ref{def:optimal_weight}. The first non-reducible term is caused by the randomness of the context and budget $\{(\bm{c}_i, \bm{B}_i)\}$. The second term decreases with time and the convergence rate depends on the randomness of the sequence and the parameter $\varsigma$ in  the Polyak-Lojasiewicz inequality.

\section{Numerical Results}\label{sec:simulation}

Weighted fairness is a classic performance metric
in the resource allocation literature \cite{MungChiang_Fairness_Infocom_2010_lan2010axiomatic}, including
fair  allocation in computer systems
\cite{resource_fairness_ghodsi2011dominant}
and 
economics \cite{hylland1979efficient}.
Here, we consider a general \emph{online} 
setting.
A total of $N$ jobs arrive sequentially, and job $t$ has a weight $c_{t}\geq0$.
The agent allocates resource $x_t\geq0$ to job $t$ at each step $t$. 
We consider the commonly-used weighted fairness  $\sum_{t=1}^Nc_t\log( x_t)$
\cite{MungChiang_Fairness_Infocom_2010_lan2010axiomatic}.
 We create the training and testing samples based on the  Azure cloud workload dataset,
which contains the average CPU reading for tasks at each step \cite{azure_cloud_trace_shahrad2020serverless}.
For detailed settings, please refer to the appendix \ref{sec:simulation_append}. 

We consider several baseline algorithms as follows. 
The Offline Optimal Oracle \textbf{OPT} is the solution to the problem in Eqn.~\eqref{eqn:target}.  We consider two  heuristics: One is Equal Resource Allocation (\textbf{Equal}) which equally allocates the total resource capacity to $N$ jobs, and another one is Resource Allocation
with Average Long-term Constraints (\textbf{\avg}) which relaxes the
inventory constraints of the weighted fairness problem
as  $\mathrm{E}_{\mathcal{P}}\left[\sum_{t=1}^N x_{t}\right]\leq B$ and uses the optimal Lagrangian multiplier for this relaxed problem
as $\lambda_t$ for online allocation.  We consider two algorithms based on dual mirror descent  which are Dual Gradient Descent (\textbf{\gradient}) and Multiplicative Weight (\textbf{\mirror} ) \cite{OCO_ODMD_balseiro2020dual}.
 To reduce the resource waste after
the last step, we slightly revise \gradient and \mirror by setting
the allocation decision for job $N$ as $\min\left(b_N,x_{\max} \right) $. 
CR-pursuit (\textbf{CR-pursuit}) is the state-of-the-art online
algorithm that makes online actions by tracking a pseudo-optimal algorithm with
a competitive guarantee
\cite{CR_persuit_lin2019competitive,OOIC-M_lin2022competitive}.
 We also compare \ouralg with the end-to-end Reinforcement Learning (\textbf{\rl}). 
A neural network with the same size of the ML model as in \ouralg is used in  
\rl to directly predict the solution $x_t$, given parameter $c_t$ and budget $b_t$ as inputs.
\begin{figure}[t]
\centering
\includegraphics[height=0.18\textwidth]{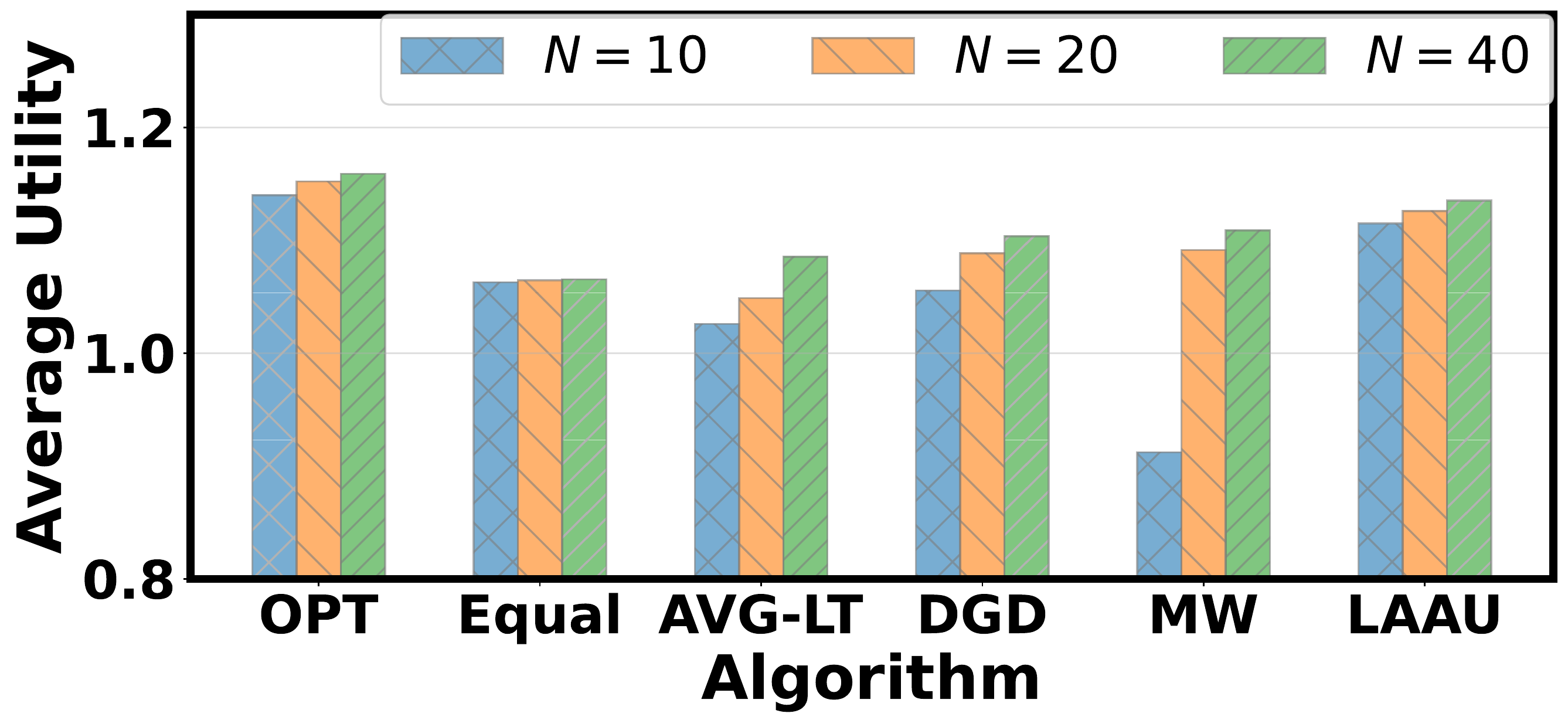}
\caption{Average utility with different episode lengths
}
\label{fig:utility_N}
\end{figure}

\textbf{Average utility.}
We first show in Fig.~\ref{fig:utility_N} the
average utilities (per time step). We do not add the average utility of CR-pursuit in the figure because its average utility for our evaluation instances is as low as 0.562, exceeding the utility range of the figure. 
 Clearly, OPT achieves the highest utility, but it is infeasible
 in practice due to the lack of complete offline information.
We can observe that the average utilities by expert algorithms including \avg, \gradient, \mirror
are even below the average utility by the simple equal allocation (Equal) when the episode length is $N=10$. This is because all the three algorithms are designed for online optimizations with long-term constraints, and not suitable for the more challenging short-term counterparts. By contrast, \ouralg performs well for all cases with large and small episode lengths, and outperforms the other algorithms designed for long-term constraints even when $N$ is as large as 40.
This demonstrates
the power of \ouralg in solving challenging online problems with inventory
constraints. More results,
including detailed comparisons 
between \ouralg and \rl, are given in  
the appendix \ref{sec:simulation_append}. 

\setlength{\columnsep}{10pt}
\begin{wrapfigure}[11]{r}{0.26\textwidth}
\vspace{-0.2cm}
\centering
	\includegraphics[width=0.26\textwidth]{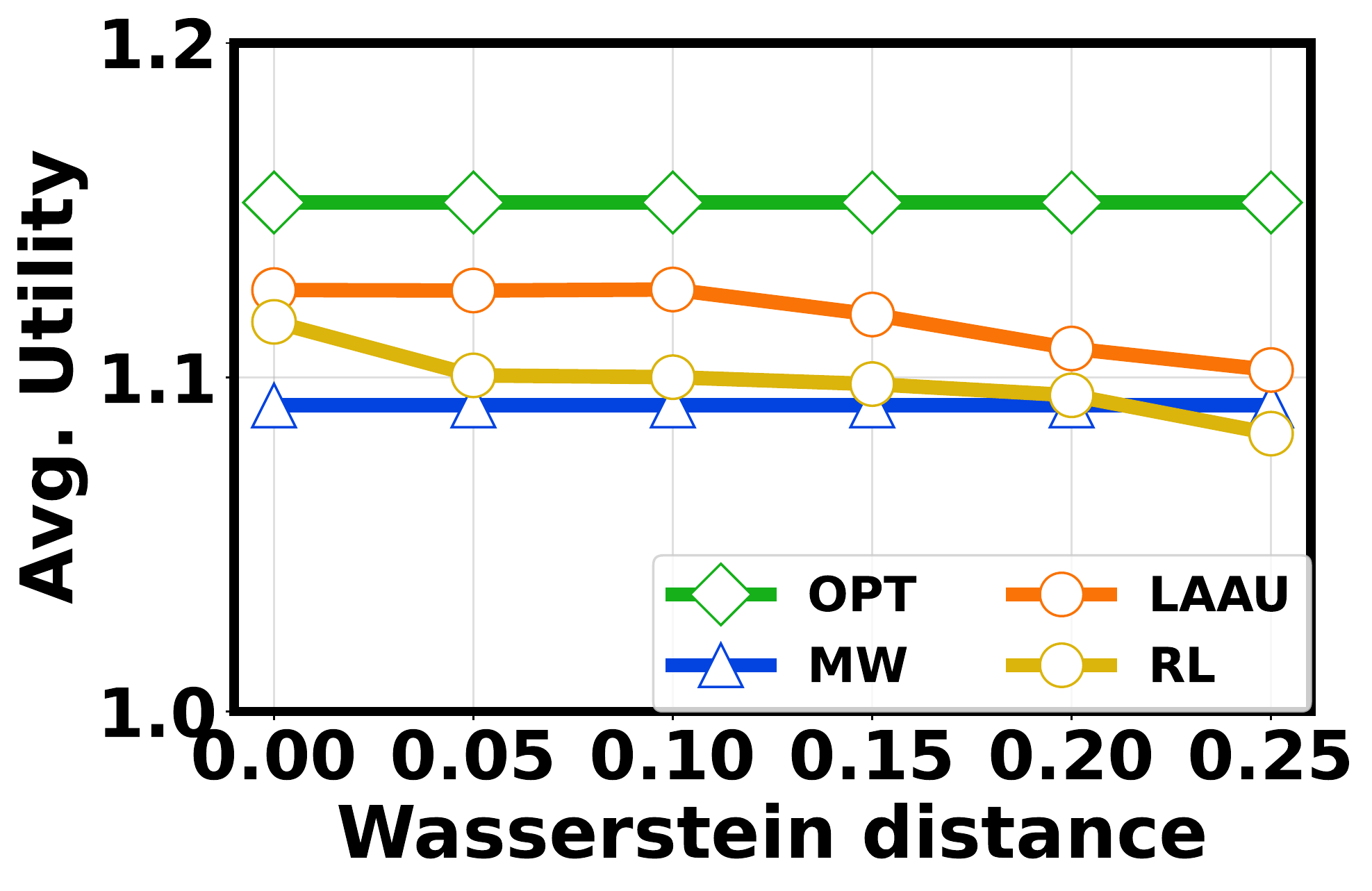}
\vspace{-0.65cm}
\caption{Average utility with different Wasserstein distances.
}
\label{fig:OOD}
\end{wrapfigure}
\textbf{OOD testing.}
In practice, 
the training-testing distributional discrepancy is common
and decreases ML model performance.
We measure the training-testing distributional difference by 
 the Wasserstein distance $d_W$. We choose the setting with episode length $N=20$ to perform the OOD evaluation.
To create the distributional discrepancy, we add i.i.d. Gaussian noise with different means and variances to the training data and keep the testing data the same as the default setting. The distributions with different Wasserstein distances is visualized by  t-SNE \cite{tsne_visualize_2008} in the appendix \ref{sec:simulation_append}.
 The offline optimal and \mirror do not make use of the training distribution, and
hence are not affected.
We can see in Fig.~\ref{fig:OOD} that, the
OOD testing decreases the performance of both \ouralg and \rl, but \ouralg is less affected by OOD testing than \rl and is still higher than that of the baseline \mirror even under large Wasserstein distance. This is because the unrolling architecture in \ouralg has an optimization layer and a budget update layer, 
which are deterministic and have no training parameters, so only the ML model to learn the Lagrangian multiplier is affected by OOD testing. Comparably,  \rl uses an parameterized end-to-end policy model trained on the offline data, so it has worse performance under OOD testing.
This highlights that by the unrolling archietecture, \ouralg has better generalization performance than end-to-end models.

\section{Conclusion}
In this paper, we focus on 
\pone
and propose a novel ML-assisted
unrolling approach based on the online decision pipeline, called \ouralg.
We derive the gradients for back-propagation 
and perform  rigorous analysis 
on the expected cost.
Finally,
 we present numerical results on weighted fairness 
and 
highlight \ouralg significantly outperforms the existing
baselines
in terms
of the average performance.

\section*{Acknowledgements}
This work was supported in part by the NSF under grant CNS-1910208.

\clearpage
\appendix
\onecolumn
\appendix

\section*{Appendix}

\section{Online Training Algorithm}\label{sec:online_training_alg}
The online training method in Section \ref{sec:online_training} is given in the following algorithm.
\begin{algorithm}[!]
	\caption{Online Training of \ouralg}
	\begin{algorithmic}[1]\label{alg:online_training}
 \STATE \textbf{Initialization:} The weight $\hat{\theta}_0$ of the unrolling model, step size $\bar{\alpha}$.
		\WHILE{a new instance $i$ arrives} 
		\STATE Perform $N$-step online inference by Algorithm \ref{alg:online_inference}.
        \STATE Collect context $\bm{c}_i$ and budget $\bm{B}_i$ of instance $i$ and update the model by 
        \[
\hat{\theta}_i=\hat{\theta}_{i-1}-\bar{\alpha}\triangledown_{\theta} L(h_{\hat{\theta}_{i-1}},\bm{c}_i),
        \]
        where $\triangledown_{\theta} L(h_{\hat{\theta}_{i-1}},\bm{c}_i)$ is obtained by back propagation in Eqn. \eqref{eqn:lossgradient}. 
	\ENDWHILE
	\end{algorithmic}
\end{algorithm}

\section{Motivating Examples}\label{sec:example}

We now present the following motivating examples to highlight
the practical relevance of \pone.

\subsection{Virtual machine resource allocation}
Server virtualization plays a fundamental role in cloud computing, reaping the multiplexing benefits and improving resource utilization \cite{VMware_DPM_Whitepaper}.
Each physical server has its own resource configuration, including
CPU, memory, I/O and storage, and is virtualized into  $N$
VMs. Jobs arrive sequentially in an online manner and each needs
to be allocated with a VM. Thus, VM resource allocation is a crucial
problem, affecting the job performance such as latency and throughput,
as well as the resource consumption.
The whole resource allocation process for $N$ VMs hosted on one
physical server can be viewed as an episode in our problem formulation,
while resource allocation for each individual VM is a per-step
online decision. The job features, e.g., data size, are
the online parameters revealed to the resource manager.
Given the online job arrivals, the goal of the agent is to optimize a certain metric of interest, such as the total throughput and fairness.
Crucially, the per-server resource constraints must be satisfied,
resulting in \emph{short}-term constraints as considered in our problem.

\subsection{Rate allocation in unreliable wireless
	networks}
Unlike wired networks, wireless networks are unreliable
and subject to intermittent availability due to various
factors such as background
interference, fading, congestion, and/or priority-based spectrum access.
Here, we use cognitive radio networks as an example,
which provide an important mechanism to
improve the utilization of increasingly wireless spectrum by
allowing unlicensed users to opportunistically access
the available spectrum unoccupied by licensed users
\cite{Cognitive_Radio_Haykin:2006:CRB:2312214.2315147}.
Due to the random time-varying spectrum usage by licensed users,
the available spectrum $\bm{B}$ for unlicensed users is thus
also random: it is available for only one fixed period of time
(i.e., an episode), but may change later.
During each time period, $N$ unlicensed users submit their
spectrum access requests to the spectrum manager
in an online manner at the moment when their own traffic becomes
available.
Each user's traffic is associated with a parameter
$c_{t}$, which indicates the $t$-th user's features
such as priority and channel condition.
The spectrum manager
makes a decision $x_{t}$, specifying the rate allocation
to user $t$. This decision results in a wireless bandwidth usage
that depends on the user's channel condition,
as well as a utility. The goal of the spectrum manager
is to maximize the total utility for these $N$ users
subject
to the total available bandwidth constraint.

\section{Simulation Settings and Additional Empirical Results}\label{sec:simulation_append}
\subsection{Experiment Setups}\label{sec:simulation_settings}
\subsubsection{Problem Setup}
In this section, we give more details about the application and simulation settings of online weighted fairness.
We consider a general \emph{online} resource allocation setting with
a total of $N$ jobs arrive sequentially, and job $t$ has a weight $c_{t}\geq0$, which is not known to the agent until the step $t$.
The agent allocates resource $x_t\geq0$ to job $t$ at each step $t$, without
knowing the weights for future jobs.
At the beginning of step $t=1$, the agent is also informed
of a total resource capacity/budget of $B$ that can be allocated over $N$ time steps.
To measure the resource allocation fairness \cite{MungChiang_Fairness_Infocom_2010_lan2010axiomatic}, we consider a commonly-used utility  $\sum_{\tau=1}^Nc_\tau\log( x_\tau)$, where $\log$ is the natural logarithm, $\bm{x}=[x_1,\cdots, x_N]$ and $\bm{c}=[c_1,\cdots, c_N]$.
Therefore, the problem of {online} resource allocation for weighted fairness with short-term capacity constraints is to maximize the fairness utility subject to the resource constraints:
\begin{equation}\label{eqn:weightedfairness_org}
	\begin{split}
	\max_{\bm{x}}\sum_{t=1}^Nc_t\log( x_t), \text{ s.t. } \sum_{t=1}^Nx_{t}\leq B
	\text{ and } x_t \in\left[x_{\min}, x_{\max}\right],
\end{split}
	\end{equation}
where $x_{\max}=40$ and $x_{\min}=1$ (to ensure a non-negative utility) represent the lowest and highest amounts of resource that can be allocated to each job.

\subsubsection{Datasets}
We perform simulations for the settings with different lengths $N$, which is chosen from $\left\lbrace 10,20,40\right\rbrace$. By default, we set $N=20$. We use the the cloud dataset \cite{azure_cloud_trace_shahrad2020serverless} for simulation. The cloud dataset is a dataset of workload sequences which are $\bm{c}$ in our formulations.   The workload sequences in the cloud dataset are divided into multiple problem instances, each with $N$ time steps, and the workload at step $t$ is $c_t$. In each instance, the total resource capacity/budget is drawn from an independent and uniform distribution in the range $[10N,15N]$, where $N$ is the number of time steps in each episode. We use 5000 offline problem instances  for training and 1000 instances for validation.
For testing, we use another 2000 problem instances from the cloud dataset to evaluate the performances of different algorithms. In the default setting, the training problem instances are drawn
from the same distribution as the testing distribution, but we will also change
the distribution when evaluating the OOD cases.
\begin{figure*}[!t]	
	\centering
	\subfigure[Default.]{
		\includegraphics[width={0.3\textwidth}]{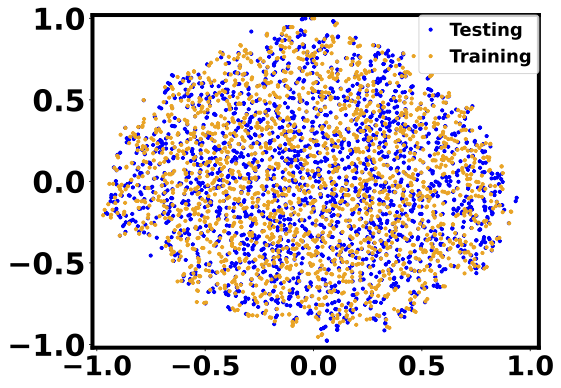}
	}
	\centering
	\subfigure[$d_W=0.05$]{
		\includegraphics[width=0.3\textwidth]{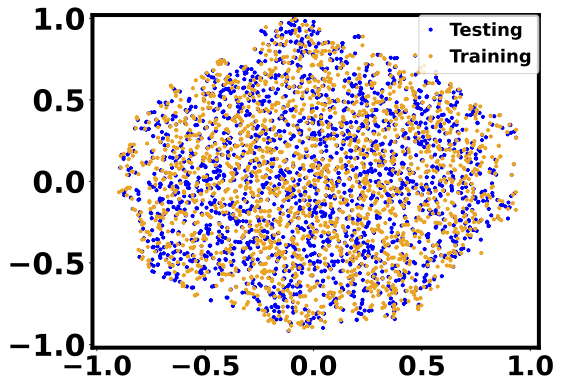}
	}%
	\centering
	\subfigure[$d_W=0.1$]{
		\includegraphics[width=0.3\textwidth]{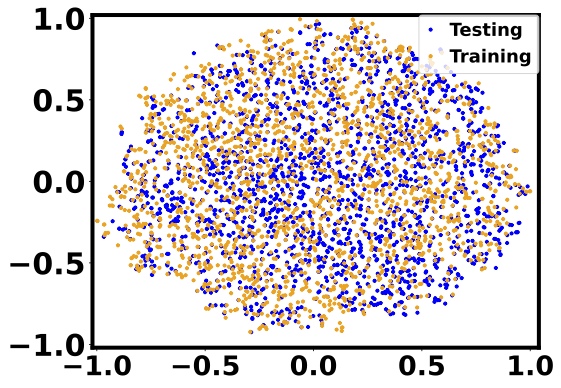}
	}%
	
	\centering
	\subfigure[$d_W=0.15$]{
		\includegraphics[width=0.3\textwidth]{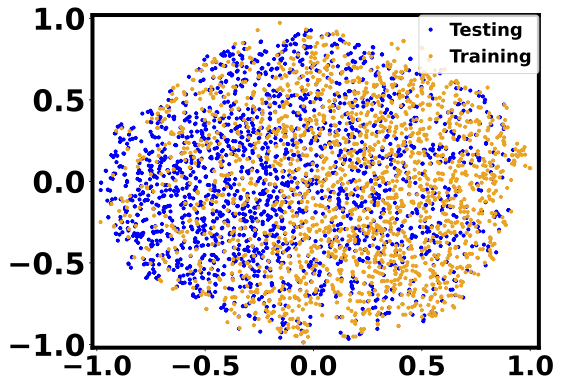}
	}%
	\centering
	\subfigure[$d_W=0.2$]{
		\includegraphics[width=0.3\textwidth]{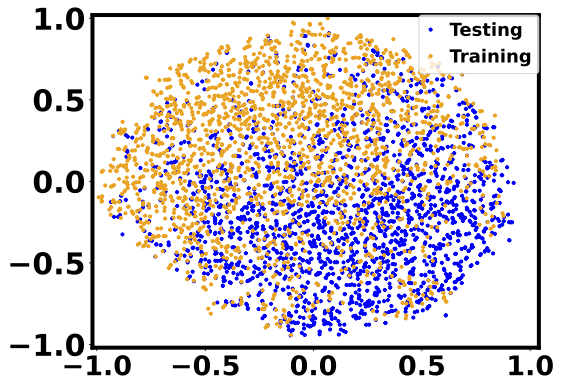}
	}%
	\centering
	\subfigure[$d_W=0.25$]{
		\includegraphics[width=0.3\textwidth]{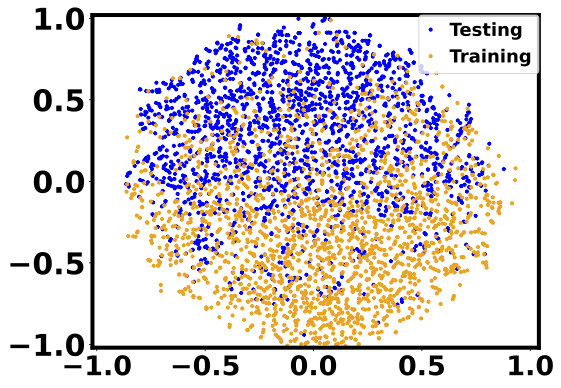}
	}%
	\vspace{-0.4cm}	
	\caption{Visualization of training-testing distribution discrepancy using t-SNE \cite{tsne_visualize_2008}. Blue and orange dots represent
		the testing and training samples, respectively.}
	\label{fig:tsne_training_testing_difference}
\end{figure*}
\subsubsection{Setup of Training}
We consider using a neural network model as the ML model in \ouralg and \rl. In both \ouralg and \rl, we use fully-connected neural networks with 2 layers,
each with 10 neurons. The neural network is trained by Adam optimizer with batch size as 10. We choose a learning rate of $5\times 10^{-3}$ for 50 epochs and a learning rate of $2.5\times 10^{-3}$ for the following 30 epochs. We run the training and testing experiments on a single CPU.

\subsection{Details of Baselines}
$\bullet$ Optimal Oracle (\textbf{OPT}):
The optimal oracle knows the complete information $(\bm{c},\bm{B})$ for
at the beginning of each episode and solves the problem in Eqn.~\eqref{eqn:target}.

$\bullet$ Equal Resource Allocation (\textbf{Equal}):
The agent equally allocates the total resource capacity to $N$ jobs,
i.e., each job receives $\frac{B}{N}$.

$\bullet$ Resource Allocation
with Average Long-term Constraints (\textbf{\avg}):
\avg minimizes the expected
cost by relaxing the
short-term capacity constraints in Eqn.~\eqref{eqn:weightedfairness_org}
as  $\mathrm{E}_{\mathcal{P}}\left[\sum_{t=1}^N x_{t}\right]\leq B$. Then, it uses the optimal Lagrangian multiplier for this relaxed problem
as $\lambda_t$ for online allocation.

$\bullet$ Online Dual Gradient Descent (\textbf{\gradient} \cite{OCO_ODMD_balseiro2020dual}):
\gradient, as one of the Dual Mirror Descent (DMD) algorithms, is designed to solve online constrained optimization with long-term constraints. It updates the Lagrangian multiplier in each step by gradient descent and then projecting the Lagrangian multiplier to the non-negative space. To make \gradient better for \pone, we slightly revise \gradient on the basis of Algorithm 1 in  \cite{OCO_ODMD_balseiro2020dual} by setting
the allocation decision for job $N$ as $\min\left(b_N,x_{\max} \right) $. The Lyapunov optimization technique \cite{Neely_Booklet} also belongs to \gradient.

$\bullet$ Online Multiplicative Weight (\textbf{\mirror} \cite{OCO_ODMD_balseiro2020dual}):
Similar as \gradient, \mirror is also a MDM method for solving online constrained optimization with long-term constraints. It also estimates the Lagrangian multiplier in an online style, but it updates the Lagrangian multiplier by multiplication \cite{OCO_MULPLICATIVE_arora2012multiplicative}.  Also, we revise \mirror to make it more suitable for OCO-SC in the same way as \gradient.

$\bullet$ Reinforcement Learning (\textbf{\rl}):\rl is a machine learning algorithm based on offline data to learn the decision at each step. \pone can be formulated as a reinforcement learning problem by viewing the remaining budget $\bm{b}_t$ and the parameter $c_t$ as states and the decision $x_t$ as the action. The reward function is the sum of the negative loss, i.e.  $-\sum_{t=1}^Nl(x_{t},c_{t})$. In \rl, we use the neural network  with the same depth and widths as \ouralg, but output of the neural network in \rl is directly the action and no optimization layer like \eqref{eqn:optlayer} is used in \rl. 

For fair comparison, in all the baselines,
if any decision results in capacity violation,
we also apply projection to the decision to meet the total
resource capacity constraint.

\subsection{Settings of OOD Testing}\label{sec:ood_illustration}
The distribution used to generate
synthetic training problem instances
may be different from the true environment distribution,
creating a  distributional discrepancy and decreasing learning-based model performance \cite{DNN_Book_Goodfellow-et-al-2016}.
We measure the training-testing distributional difference in terms
of the Wasserstein distance $d_W$.
To create the distributional difference, we add i.i.d. Gaussian noise with different means and variances to the training data in the default setting.

We visualize in Fig.~\ref{fig:tsne_training_testing_difference} the training-testing distributional differences
using  t-SNE (which converts high-dimensional distribution
into a 2-dimensional one for visualization) \cite{tsne_visualize_2008} under different
Wasserstein distances. We can see that when $d_W=0.25$, the distributional difference
is fairly large, with well separated testing and training samples.
\subsection{Further Results}

\begin{figure}[!t]
\minipage{0.42\textwidth}
\vspace{-0.26cm}
  \includegraphics[width=\linewidth]{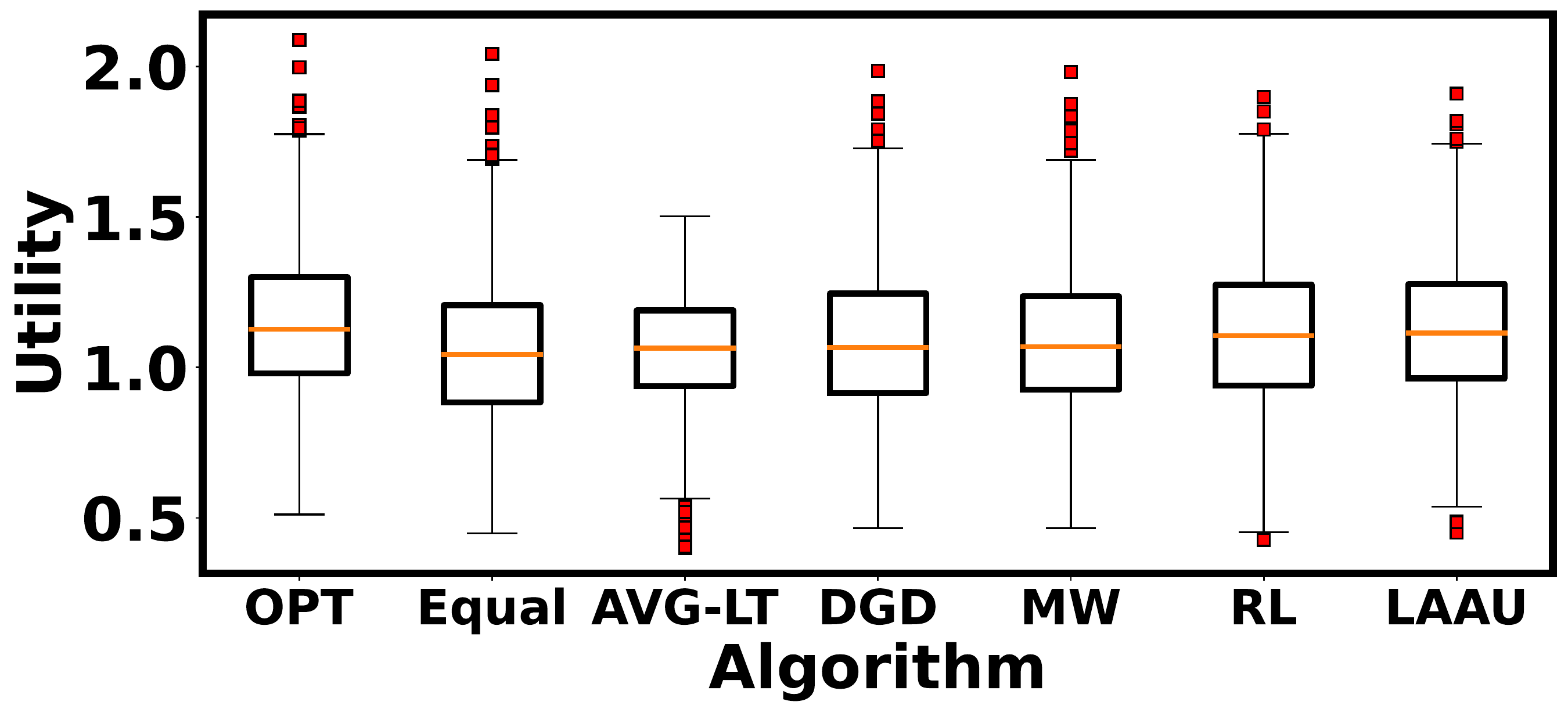}
  \caption{Distribution of Average utility (episode length $N=20$).}\label{fig:utility_N_box}
\endminipage\hfill
\minipage{0.285\textwidth}
  \includegraphics[width=\linewidth]{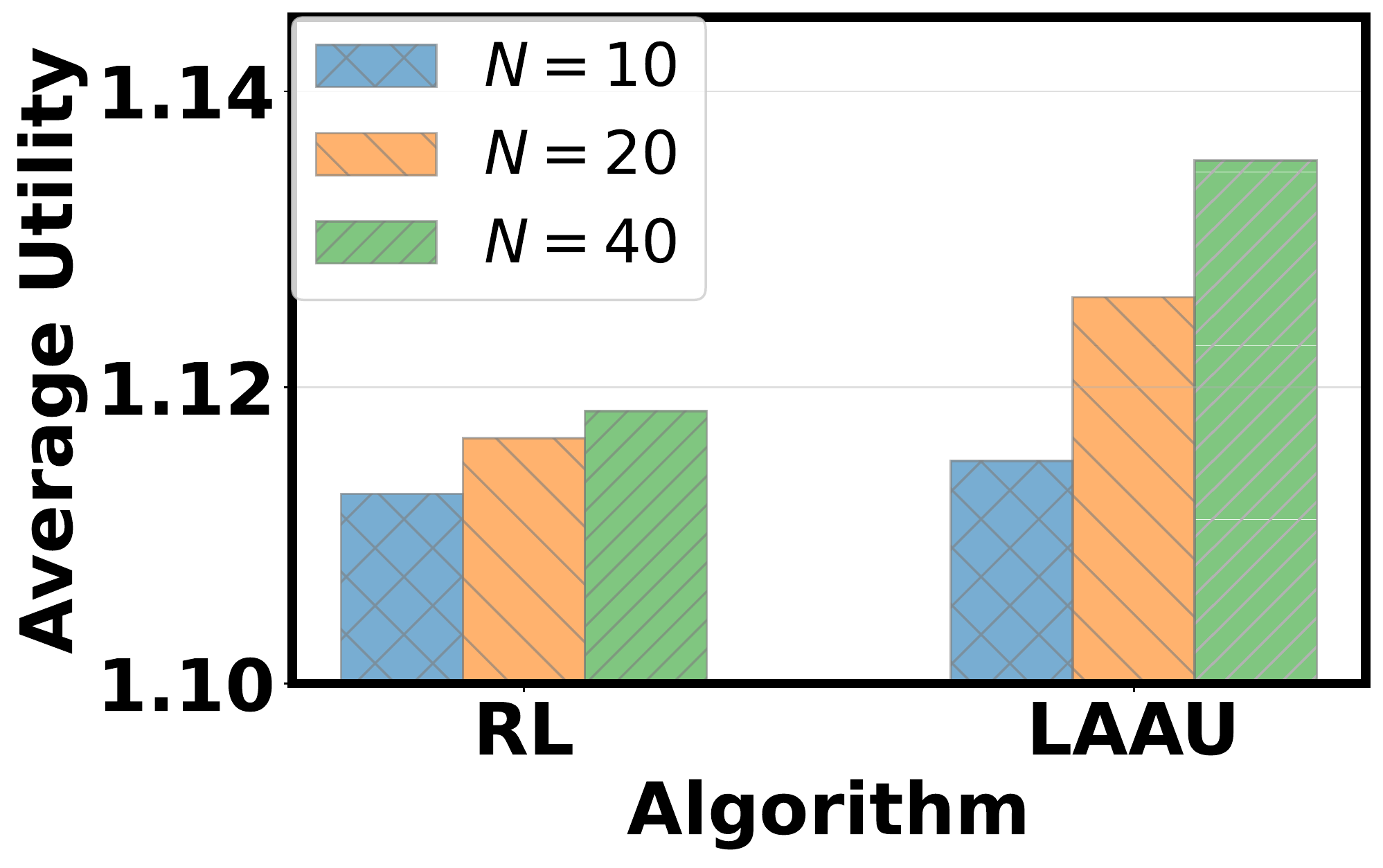}
  \caption{Average utility with different episode lengths.}\label{fig:rl_laau}
\endminipage\hfill
\minipage{0.255\textwidth}%
  \includegraphics[width=\linewidth]{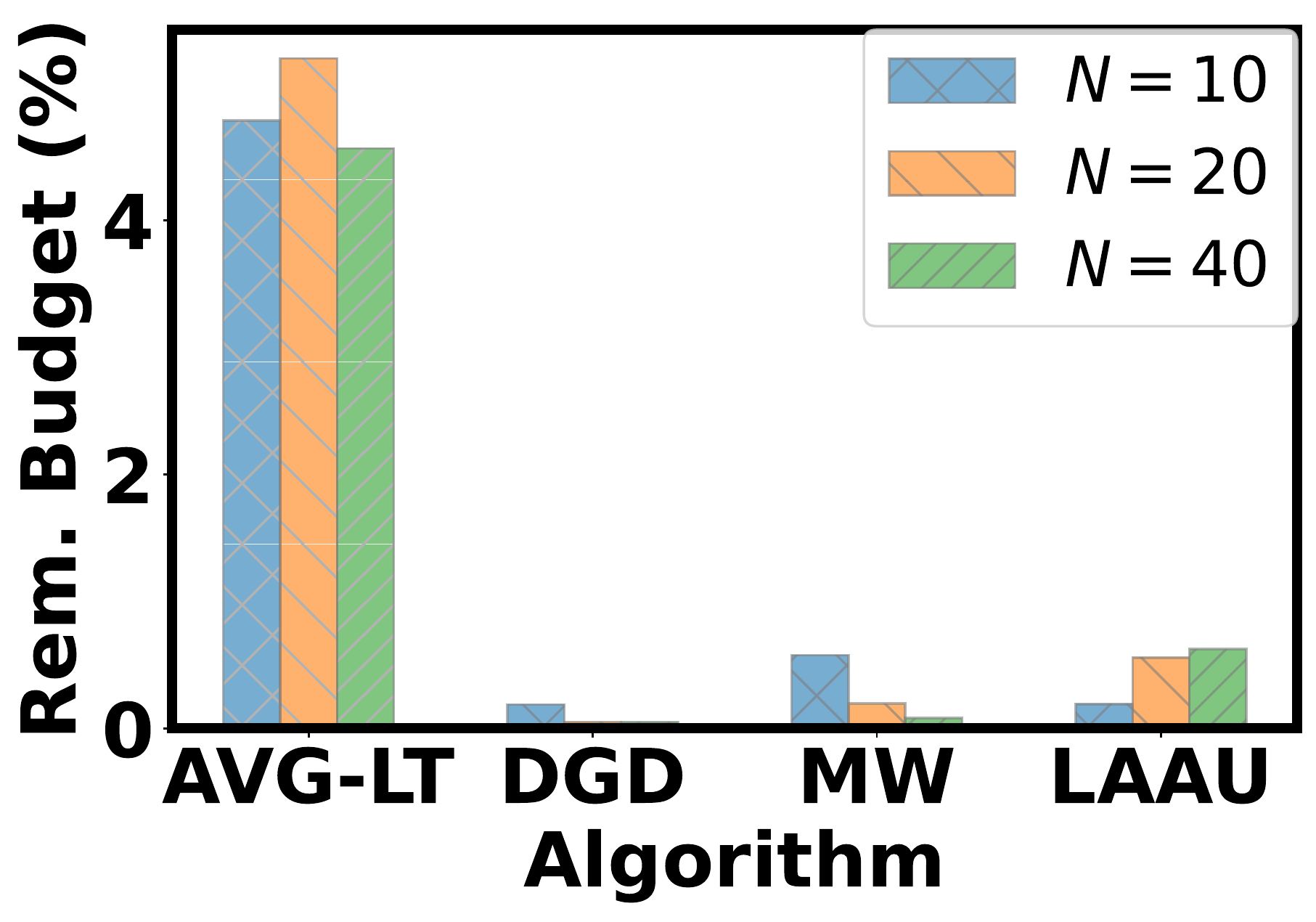}
  \caption{Average remaining budget}\label{fig:remaining_budget}
\endminipage
\end{figure}
\subsubsection{\rl and \ouralg with Different Lengths $N$}\label{sec:rl_laau}
In Figure \ref{fig:rl_laau}, we compare the average utilities of \rl and \ouralg in terms of different number of jobs $N$. We find that when $N$ is larger, \ouralg outperforms \rl more. This is because the unrolling architecture of \ouralg exploits the knowledge from the optimization layers \eqref{eqn:optlayer} and hence benefits generalization. When $N$ increases, \ouralg uses more optimization layers while \rl uses more neural layers, making the advantage of \ouralg compared with \rl become larger. This shows the benefit of the unrolling architecture in reducing generalization error.

\subsubsection{Distribution of Average Utility}
In Figure \ref{fig:utility_N_box}, we show the distribution of utility when the episode length $N=20$. For each algorithm, we give the median, 25 and 75 percentile, maximum and minimum utility, and the outliers. The results show that \ouralg has the best median and 25\% percentile tail utility among all the other baselines except OPT. In addition, the baseline algorithms do not have a substantial advantage over \ouralg in terms of the minimum utility.

\subsubsection{Remaining Resource}

We show in Fig.~\ref{fig:remaining_budget} the average
remaining resources for different algorithms.
OPT and Equal use up all the resources and hence is not shown.
We can find that the remaining resource of \avg is the highest because the Lagrangian multiplier of \avg is constant and cannot be updated according to the input parameters and remaining resource budgets online.  The algorithms based on DMD, i.e. \gradient and \mirror, both have low remaining budgets. 
This is because they aim to converge to the optimal Lagrangian multiplier that optimizes the dual problem of the original problem, 
which can guarantee a low complementary slackness (sum of the products of the Lagrangian multiplier and the remaining resource), as proved in Proposition~5 of \cite{OCO_ODMD_balseiro2020dual}. 
\ouralg, albeit having a slightly higher remaining resource than  \gradient and \mirror when $N$ is large, still achieves a sufficiently low remaining resource (less
than 1\%). The reason for remaining resources is
 the per-step allocation constraint.
Importantly, \ouralg has higher average utility with slightly lower utilized resource capacities than DMD, which further demonstrates the advantages of \ouralg.

\section{Proofs of Conclusions in Section \ref{sec:training}}\label{sec:proofbackpropagation_section5}
In this section ,we first drive the gradients in Proposition \ref{lma:grad_optlayer} and then prove the conditions for differentiable KKT in Proposition \ref{lma:condition_optlayer}.
First, the notations are given as follows.\\
\textbf{Notations:} 
 The gradient of $\bm{x}\in\mathbb{R}^d$ regarding $\bm{y}\in\mathbb{R}^q$ is denoted as $\bigtriangledown_{\bm{y}}\bm{x}\in\mathbb{R}^{d\times q}$.  $\bigtriangledown_{\bm{x}}f(\bm{x})\in\mathbb{R}^{d}$ is the gradient of the function $f(\bm{x})$  regarding $\bm{x}\in\mathbb{R}^d$.
$\bigtriangledown_{\bm{x}\bm{x}}f(\bm{x})\in\mathbb{R}^{d\times d}$ is the Hessian matrix of  $f(\bm{x})$.
 $\mathrm{E}[\,\cdot\,]$ means taking expectation. $\mathrm{diag}(\bm{x})$ returns a diagonal matrix with the diagonal elements from $\bm{x}$.  
\subsection{Proof of Proposition \ref{lma:grad_optlayer}}\label{sec:proof_grad_opt}
\begin{proof}
	To derive the gradients of the optimization layer, we start from the KKT conditions of the corresponding constrained optimization \eqref{eqn:optlayer} as below.
	\begin{equation}\label{eqn:kkt}
		\left\{\begin{matrix}
			\bigtriangledown_{x_t}l\left( x_t, c_t\right)+\sum_{m=1}^M\left(  \lambda_{m,t} +\mu_{m,t}\right)  \bigtriangledown_{x_t}g_m\left(x_{t},c_{t} \right) =\bm{0}_{1\times d}\\
			\mu_{t}\bigodot\left( \bm{g}\left(x_{t},c_{t} \right)-\bm{b}_{t} \right) =\bm{0}_{M}\\
			\bm{g}\left(x_{t},c_{t} \right)\leq \bm{b}_{t}, \mu_t\geq 0.
		\end{matrix}\right.,
	\end{equation}
	where $\bigodot$ is the Hadamard product, $\mu_{t}\in\mathbb{R}^M$ is the optimal dual variable for the constrained optimization \eqref{eqn:optlayer}, the first equation is called stationarity, the second equation is called complementary slackness, and the two inequalities are primal and dual feasibility, respectively.  Note that both the optimal primal variable $x_t$ and the optimal dual variable $\mu_t$ rely on the inputs of \eqref{eqn:optlayer} including $\lambda_t$ and $\bm{b}_t$.
	
	We first derive the gradient of $x_t$ with respect to $\lambda_t$.
	Taking gradient for both sides of the first two equations in \eqref{eqn:kkt} with respect to $\lambda_t$, we get
	\begin{equation}\label{eqn:kkt_1am}
		\begin{split}
			&\begin{bmatrix}
				\bigtriangledown_{x_t x_t}\!l\left( x_t, c_t\right)+\sum_{m=1}^M\!\!\left(  \lambda_{m,t} +\mu_{m,t}\right) \!\!\bigtriangledown_{x_t x_t}\!g_m\!\left(x_{t},c_{t} \right)& \left[ \bigtriangledown_{x_{t}}\bm{g}\left(x_{t},c_{t} \right)\right]^\top \\
				\left[ \mu_{1,t}\left[ \bigtriangledown_{x_{t}} g_1\left(x_{t},c_{t} \right)\right]^\top ,\cdots, \mu_{M,t}\!\left[ \bigtriangledown_{x_{t}} g_M\left(x_{t},c_{t} \right) \right]^\top\right] ^\top& \mathrm{diag}\left( \bm{g}\left(x_{t},c_{t} \right)-\bm{b}_{t}\right)
			\end{bmatrix}
			\begin{bmatrix}
				\bigtriangledown_{\lambda_t}x_{t}\\
				\bigtriangledown_{\lambda_t}\mu_{t}
			\end{bmatrix}\\
			&+
			\begin{bmatrix}
				\left[ \bigtriangledown_{x_{t}}\bm{g}\left(x_{t},c_{t} \right)\right]^\top \\
				\bm{0}_{M\times M}
			\end{bmatrix}=\bm{0}_{(d+M)\times M}.
		\end{split}
	\end{equation}
	Denote the first matrix in Eqn.~\eqref{eqn:kkt_1am} as $\Delta$, and its four blocks are $\Delta_{11}=\bigtriangledown_{x_t x_t}\!l\left( x_t, c_t\right)+\sum_{m=1}^M\!\!\left(  \lambda_{m,t} +\mu_{m,t}\right) \!\!\bigtriangledown_{x_t x_t}\!\!g_m\!\left(x_{t},c f_{t} \right)$, $\Delta_{12}=\left[ \bigtriangledown_{x_{t}}\bm{g}\left(x_{t},c_{t} \right)\right]^\top $, $\Delta_{21}=	\left[ \mu_{1,t}\!\left[ \bigtriangledown_{x_{t}} g_1\left(x_{t},c_{t} \right)\right]^\top ,\cdots, \mu_{M,t}\!\left[ \bigtriangledown_{x_{t}} g_M\left(x_{t},c_{t} \right) \right]^\top\right] ^\top$ and $\Delta_{22}=\mathrm{diag}\left( \bm{g}\left(x_{t},c_{t} \right)-\bm{B}_{t}\right)$. Also, we denote the Schur-complement of $\Delta_{11}$ in the matrix $\Delta$ as $\mathrm{Sc}\left(\Delta, \Delta_{11}\right)=\Delta_{22}-\Delta_{21}\Delta^{-1}_{11}\Delta_{12} $. If $\mathrm{Sc}\left(\Delta, \Delta_{11}\right)$ is invertible (we give the conditions for this in Proposition \ref{lma:condition_optlayer}),  by block-wise matrix inverse, we can solve the gradient $\bigtriangledown_{\lambda_t}x_{t}$ as
	\begin{equation}\label{eqn:gradopt_lambda}
		\bigtriangledown_{\lambda_t}x_{t}= -\left(\Delta_{11}^{-1}+\Delta_{11}^{-1}\Delta_{12}\mathrm{Sc}\left(\Delta, \Delta_{11}\right)^{-1}\Delta_{21}\Delta_{11}^{-1} \right) \Delta_{12}.
	\end{equation}

Next, we drive the gradient of $x_t$ with respect to $\mathbf{b}_t$. Similarly, Taking gradient for both sides of the first two equations in \eqref{eqn:kkt} with respect to $\bm{b}_t$, we get
	\begin{equation}\label{eqn:kkt_b}
	\begin{split}
		&\begin{bmatrix}
			\bigtriangledown_{x_t x_t}l\left( x_t, c_t\right)+\sum_{m=1}^M\!\left(  \lambda_{m,t} +\mu_{m,t}\right) \!\bigtriangledown_{x_t x_t}\!g_m\!\left(x_{t},c_{t} \right)& \left[ \bigtriangledown_{x_{t}}\bm{g}\left(x_{t},c_{t} \right)\right]^\top \\
			\left[ \mu_{1,t}\!\left[ \bigtriangledown_{x_{t}} g_1\left(x_{t},c_{t} \right)\right]^\top ,\cdots, \mu_{M,t}\!\left[ \bigtriangledown_{x_{t}} g_M\left(x_{t},c_{t} \right) \right]^\top\right] ^\top& \mathrm{diag}\left( \bm{g}\left(x_{t},c_{t} \right)-\bm{b}_{t}\right)
		\end{bmatrix}
		\begin{bmatrix}
			\bigtriangledown_{\bm{b}_t}x_{t}\\
			\bigtriangledown_{\bm{b}_t}\mu_{t}
		\end{bmatrix}\\
		&+
		\begin{bmatrix}
		    \bm{0}_{d\times M}\\
			-\mathrm{diag}(\mu_t)
		\end{bmatrix}=\bm{0}_{(d+M)\times M}.
	\end{split}
\end{equation}
Thus by solving Eqn.\eqref {eqn:kkt_b} with block-wise matrix inverse lemma, we have
\begin{equation}
\bigtriangledown_{\bm{b}_t}x_{t}= -\Delta_{11}^{-1}\Delta_{12}\mathrm{Sc}\left(\Delta, \Delta_{11}\right)^{-1}\mathrm{diag}(\mu_t).
\end{equation}
\end{proof}

\subsection{Proof of Proposition ~\ref{lma:condition_optlayer}}\label{sec:proof_grad_condition}
\begin{proof}
	By the proof in Section \ref{sec:proof_grad_opt} to derive the gradients in Proposition~\ref{lma:grad_optlayer},  we can find that if $\Delta_{11}$ and its Schur-complement $\mathrm{Sc}(\Delta,\Delta_{11})$ are invertible and the norm of $\Delta_{12}$ and $\Delta_{21}$ is finite, then KKT based differentiation is viable.
	
	First, if the norms of the constraint functions $\left\| \bigtriangledown_{x_t}g_m\left( x_t, c_t\right)\right\|_2$ are bounded, then  the norm of $\Delta_{12}$ and $\Delta_{21}$ is finite.
	Then, we give the condition that $\Delta_{11}$ is invertible. Since $\Delta_{11}$ is the Hessian matrix of the Lagrangian relaxed objective $l\left( x_t, c_t\right)+\sum_{m=1}^M\!\!\left(  \lambda_{m,t} +\mu_{m,t}\right) g_m\left(x_{t},\bm{c}_{t} \right)$, $\Delta_{11}$ is revertible only if it is positive definite. Given that $\lambda_{m,t}$ which is the output of machine learning model can be forced to be positive, we require that the loss function $l$ or any constraint function $g_m,m=1,\cdots, M$ is strongly convex with respect to $x_t$, which guarantees that $\bigtriangledown_{x_t,x_t}l\left( x_t, c_t\right)$ or $\bigtriangledown_{x_t,x_t}g_m\left(x_{t},\bm{c}_{t} \right)$ for any $m=1\dots, M$ is positive definite.

		Finally, we prove the conditions can guarantee that $\mathrm{Sc}(\Delta,\Delta_{11})$ is invertible.  Here for ease of notations, we omit the time step indices and the inputs of $g_m(x_t,c_t)$ and write $g_m(x_t,c_t)$  as $g_m$, $\mu_{m,t}$ as $\mu_m$. Remember that $\mathcal{A}$ is the index set of constraints that are activated. According to complementary slackness in KKT conditions \eqref{eqn:kkt}, we know that if $m\notin \mathcal{A}$, i.e. $g_m<b_{m}$, then $\mu_{m}=0$. Here we denote $|\mathcal{A}|=M'$, i.e. $M'$ constraints are activated. To prove the condition, we reorder the $M$ constraints, change the indices of them and let $g_{m}<b_m$ and $\mu_m=0$ for $m=1\cdots, M-M'$. Now $\mathcal{A}=\{M-M'+1,\cdots, M\}$. From the conditions that $\mu_m\neq0$ if $m\in\mathcal{A}$, we have  $\mu_m>0$ and $g_{m}=b_m$  for $m=M-M'+1,\cdots, M$. In other words, the first $M-M'$ inequalities are un-activated and the last $M'$ inequalities are activated. The Shur-complement of $\Delta_{11}$ can now be written as
	\begin{equation}\label{eqn:proofgradcondition1}
		\begin{split}
			&\mathrm{Sc}\left(\Delta, \Delta_{11}\right)\\
			=&\begin{bmatrix}	
			\mathrm{diag}(\left[\bm{g}-\bm{b}\right]_{1:M-M'})& \bm{0}_{(M-M')\times (M-M')}\\
				-\left[\Delta_{21} \right]_{M-M'+1:M}\Delta_{11}^{-1}\left[\Delta_{12} \right]_{:,1:M-M'} & -\left[\Delta_{21} \right]_{M-M'+1:M}\Delta_{11}^{-1}\left[\Delta_{12} \right]_{:,M-M'+1:M}
			\end{bmatrix}.
		\end{split}
	\end{equation}
where $\bm{X}_{a:b}$ denotes the sub-matrix concatenated by rows of matrix $\bm{X}$ with indices $a,a+1,\cdots, b$,  $\bm{X}_{:,a:b}$ denotes the sub-matrix concatenated  columns of matrix $\bm{X}$ with indices $a,a+1,\cdots, b$. By Block matrix determinant lemma, we have
\begin{equation}
	\mathrm{det}\left(\mathrm{Sc}\left(\Delta, \Delta_{11}\right)\right) =\mathrm{det}\left(\mathrm{diag}(\left[\bm{g}-\bm{b}\right]_{1:M-M'}) \right) \mathrm{det}\left( -\left[\Delta_{21} \right]_{M-M'+1:M}\Delta_{11}^{-1}\left[\Delta_{12} \right]_{:,M-M'+1:M}\right).
	\end{equation}
We have $\mathrm{det}\left(\mathrm{diag}(\left[\bm{g}-\bm{b}\right]_{1:M-M'}) \right)\neq 0$ since $g_{m}<b_m$  for $m=1\cdots, M-M'$. Now it remains to prove $\left[\Delta_{21} \right]_{M-M'+1:M}\Delta_{11}^{-1}\left[\Delta_{12} \right]_{:,M-M'+1:M}$ has full rank. To do that, we can rewrite the matrix as
\begin{equation}
	\left[\Delta_{21} \right]_{M-M'+1:M}\Delta_{11}^{-1}\left[\Delta_{12} \right]_{:,M-M'+1:M}=\mathrm{diag}\left( \mu_{\mathrm{act}}\right)  \bigtriangledown_{x_t}\!\!\bm{g}_{\mathrm{act}} \Delta_{11}^{-1}\bigtriangledown_{x_t}\!\!\bm{g}^\top_{\mathrm{act}},
	\end{equation}
where $\mu_{\mathrm{act}}=\left[ \mu_{M-M'+1},\cdots, \mu_M\right] $ is a non-zero vector according to the condition, and $\bigtriangledown_{x_t}\bm{g}_{\mathrm{act}}=\left[\bigtriangledown_{x_t}\bm{g}\right]_{M-M'+1:M}$ is the gradient matrix regarding the activated constraint functions. By the condition that the gradients regarding the activated constraint functions are linear dependent and $d\geq M'$, we have $\mathrm{rank}(\bigtriangledown_{x_t}\bm{g}_{\mathrm{act}})=M'$. We have proved that $\Delta_{11}$ is invertible by our conditions, so we have $\mathrm{rank}(\Delta_{11}^{-1})=d$ and $\mathrm{rank}(\Delta_{11}^{-1/2})=d$ with $\mathrm{rank}(\Delta_{11}^{-1})=\mathrm{rank}(\Delta_{11}^{-1/2})\mathrm{rank}(\Delta_{11}^{-1/2})$. Since for a real matrix $\bm{X}$, $\mathrm{rank}(\bm{X}^\top\bm{X})=\mathrm{rank}(\bm{X})$, and a matrix and its Gram matrix have the same rank, we have
\begin{equation}
	\mathrm{rank}\left(\bigtriangledown_{x_t}\bm{g}_{\mathrm{act}} \Delta_{11}^{-1}\bigtriangledown_{x_t}\bm{g}^\top_{\mathrm{act}} \right) =\mathrm{rank}\left( \bigtriangledown_{x_t}\bm{g}_{\mathrm{act}} \Delta_{11}^{-1/2}\right) =M',
	\end{equation}
where the  second equality holds due to Sylvester's rank inequality such that $\mathrm{rank}\left( \bigtriangledown_{x_t}\bm{g}_{\mathrm{act}} \Delta_{11}^{-1/2}\right) \geq M'$ and the fact that $\mathrm{rank}\left( \bigtriangledown_{x_t}\bm{g}_{\mathrm{act}} \Delta_{11}^{-1/2}\right)\leq \min\left(\mathrm{rank}\left( \bigtriangledown_{x_t}\bm{g}_{\mathrm{act}}\right), \mathrm{rank}\left(\Delta_{11}^{-1/2}\right)  \right)=M' $. Since $\mu_{\mathrm{act}}>0$, we have
\begin{equation}
	\mathrm{det}\left(\mathrm{Sc}\left(\Delta, \Delta_{11}\right)\right)=\mathrm{det}\left(\mathrm{diag}(\left[\bm{g}-\bm{b}\right]_{1:M-M'}) \right)\mathrm{det}\left( \mathrm{diag}\left( \mu_{\mathrm{act}}\right) \right)  \mathrm{det}\left( \bigtriangledown_{x_t}\bm{g}_{\mathrm{act}} \Delta_{11}^{-1}\bigtriangledown_{x_t}\bm{g}^\top_{\mathrm{act}}\right) >0,
\end{equation}
Thus $\mathrm{Sc}\left(\Delta, \Delta_{11}\right) $ has full rank and is invertible, which completes the proof.

\end{proof} 
\section{Proofs of Results in Section \ref{sec:analysis}}\label{sec:generalization_proof}
\subsection{Proof of Theorem \ref{thm:generalization}}
\begin{lemma}[Sensitivity of the total loss.]\label{lma:Lipschitzofloss}
	If  the ML model
	$f_{\theta}$ is Lipschitz continuous with respect to the resource capacity input $\bm{B}$ and parameter input
	$\bm{c}$ are bounded, and the conditions in Proposition \ref{lma:condition_optlayer} are satisfied, then the loss $L(\bm{h}_\theta)$ is Lipschitz continuous with respect to the context sequence $\bm{c}$ and total budget $\bm{B}$, i.e.
	\[
	L(\pi_\theta,(\bm{c},\bm{B}))-L(\pi_\theta, (\bm{c}',\bm{B}'))\leq \Gamma_{L,c}\omega_c+\Gamma_{L,b}\omega_b ,
	\]
	where $\omega_c=\max_{\bm{c},\bm{c}'\in \mathbb{C}}\left\|\bm{c}-\bm{c}' \right\| $ is the size of the parameter space $\mathbb{C}$, $\omega_b=\max_{\bm{B},\bm{B}'\in \mathbb{B}}\left\|\bm{B}-\bm{B}' \right\| $ is the size of the resource capacity space $\mathbb{B}$, and $\Gamma_{L,c}$ and $\Gamma_{L,b}$ are bounded.
\end{lemma}
\begin{proof}
	We first bound the Lipschitz constants of the optimization layer $p(c_t,\lambda_t,\bm{b}_t)$ with respect to $\lambda_t$,$\bm{b}_t$, and $c_t$. Taking $l_2$-norm for both sides of Eqn.~\eqref{eqn:gradopt_lambda} and using triangle inequality, we have
	\begin{equation}\label{eqn:boundgradlambda}
		\begin{split}
			\left\| \bigtriangledown_{\lambda_t}\!x_{t}\right\|_2 &\leq \left\| \Delta_{11}^{-1}\Delta_{12}\right\|_2 +\left\| \Delta_{11}^{-1}\Delta_{12}\mathrm{Sc}\left(\Delta, \Delta_{11}\right)^{-1}\Delta_{21}\Delta_{11}^{-1}\Delta_{12}\right\|_2 \\
			&\leq \left\| \Delta_{11}^{-1} \Delta_{12}\right\|_2+\left\| \Delta_{11}^{-1}\Delta_{12}\right\|_2 \left\| \mathrm{Sc}\left(\Delta, \Delta_{11}\right)^{-1}\right\|_2 \left\| \Delta_{21}\right\|_2 \left\| \Delta_{11}^{-1}\Delta_{12}\right\|_2,
		\end{split}
	\end{equation}
	where the second inequality holds because any induced operator norm is sub-multiplicative.
	
	Denote the Lipschitz constants for $l$ and $g_m$ for any $ m=1,\cdots, M$ are $\Gamma_l$ and $\Gamma_g$ respectively. Then we have $\left\| \bigtriangledown_{x_t}l\left( x_t, c_t\right)\right\|_2\leq \Gamma_{l,x}$ and $\left\| \bigtriangledown_{x_t}g_m\left( x_t, c_t\right)\right\|_2\leq \Gamma_{g,x}$ for any $ m=1,\cdots, M$. Thus we have $\left\| \Delta_{12}\right\|_F=\left\| \bigtriangledown_{x_{t}}\!\bm{g}\left(x_{t},\bm{c}_{t} \right)\right\|_F\leq \sqrt{M}\Gamma_{g,x} $ and $\left\| \Delta_{21}\right\|_F\leq \bar{\mu}\sqrt{M}\Gamma_{g,x}$ where $\bar{\mu}$ is the largest dual variable for any $m=1,\cdots M$ and $t=1,\cdots, N$. Also, we have the norm of Hessian matrices bounded as $0\leq\breve{\rho}_{l,x,x} \leq \left\|\bigtriangledown_{x_t x_t}\!l\left( x_t, c_t\right) \right\|_2\leq \bar{\rho}_{l,x,x}$ and $0 \leq \breve{\rho}_{g,x,x} \leq \left\| \bigtriangledown_{x_t x_t}g_m\left(x_{t},\bm{c}_{t} \right)\right\|_2\leq \bar{\rho}_{g,x,x}$ by the convexity of $l$ and convexity of $g_m$. Since either $l$ or one of $g_m, m=1,\cdots, M$ is strongly convex,  the least singular value of $\Delta_{11}$ is larger than or equal to $\breve{\rho}_{l,g}=\min\left\lbrace \breve{\rho}_{l,x,x}, \breve{\lambda}\breve{\rho}_{g,x,x}\right\rbrace $ where $\breve{\lambda}$ is the lower bounds of $\lambda_{m,t}$. Thus, we have $\left\| \Delta_{11}^{-1}\right\|_2 \leq \frac{1}{\breve{\rho}_{l,g}}$. Thus we have $\left\| \Delta_{11}^{-1}\Delta_{12}\right\|_2\leq \left\| \Delta_{11}^{-1}\right\|_2 \left\| \Delta_{12}\right\|_F\leq \sqrt{M} \frac{\Gamma_{g,x}}{\breve{\rho}_{l,g}}$.
	
	If the conditions in Proposition \ref{lma:condition_optlayer} are satisfied, by the proof of Proposition \ref{lma:condition_optlayer},  $\mathrm{Sc}\left(\Delta, \Delta_{11}\right)^{-1}$ has full rank and $\left\| \mathrm{Sc}\left(\Delta, \Delta_{11}\right)^{-1}\right\|_2$ is upper bounded.
	By Eqn.~\eqref{eqn:proofgradcondition1}, the smallest singular value of $\mathrm{Sc}\left(\Delta, \Delta_{11}\right)^{-1}$ is expressed as
	\begin{equation}
		\sigma_{\min}\left( \mathrm{Sc}\left(\Delta, \Delta_{11}\right)\right) =\min\left\lbrace \min \left\lbrace \left| g_m-b_m\right|,m\notin\mathcal{A} \right\rbrace,\sigma_{\min}\left(\left[\Delta_{21} \right]_{\mathcal{A}}\Delta_{11}^{-1}\left[\Delta_{12} \right]_{\mathcal{A}} \right)  \right\rbrace,
	\end{equation}
	where $\sigma_{\min}\left( \cdot\right) $ returns the smallest singular value of a matrix and $\bm{X}_{\mathcal{A}}$ is a sub-matrix of $\bm{X}$ with indices in $\mathcal{A}$. For ease of notation, we denote $\sigma_{p,\lambda}=\sigma_{\min}\left(\left[\Delta_{21} \right]_{\mathcal{A}}\Delta_{11}^{-1}\left[\Delta_{12} \right]_{\mathcal{A}} \right)$. Also, we denote $\mathrm{gap}(\bm{g},\bm{b})=\min \left\lbrace \left| g_m-b_m\right|,m=1,\cdots, M' \right\rbrace$. Then, we have
	\begin{equation}
		\left\| \mathrm{Sc}^{-1}\left(\Delta, \Delta_{11}\right)\right\|_2  =\sigma^{-1}_{\min}\left( \mathrm{Sc}\left(\Delta, \Delta_{11}\right)\right) =1/\min\left\lbrace \mathrm{gap}(\bm{g},\bm{b}), \sigma_{p,\lambda}\right\rbrace.
	\end{equation}
	Combining with inequality \eqref{eqn:boundgradlambda}, we have
	\begin{equation}
		\left\| \bigtriangledown_{\lambda_t}\!x_{t}\right\|_2 \leq \sqrt{M}\frac{\Gamma_{g,x}}{\breve{\rho}_{l,g}}+\frac{M^{3/2}\bar{\mu}\Gamma_{g,x}^3}{\breve{\rho}_{l,g}^2\min\left\lbrace d(\bm{g},\bm{b}), \sigma_{p,\lambda}\right\rbrace}.
	\end{equation}
	Similarly for the gradient of $p$ with respect to $\bm{b}_t$,
	And  the corresponding bound is
	\begin{equation}\label{eqn:gradopt_budget_bound}
		\begin{split}
			\left\| \bigtriangledown_{\bm{b}_t}x_{t}\right\|_2 &\leq \left\| \Delta_{11}^{-1}\Delta_{12}\right\|_2\left\|  \mathrm{Sc}\left(\Delta, \Delta_{11}\right)^{-1}\right\|_2 \left\| \mathrm{diag}(\mu_t)\right\|_2 \\
			&\leq \frac{\sqrt{M}\Gamma_{g,x}\bar{\mu}}{\breve{\rho}_{l,g}\min\left\lbrace d(\bm{g},\bm{b}), \sigma_{p,\lambda}\right\rbrace}.
		\end{split}
	\end{equation}
	And the gradient of $p$ with respect to $c_t$ can also be derived by KKT as
	\begin{equation}\label{eqn:gradopt_c}
		\begin{split}
			&\bigtriangledown_{c_t}\bm{x}_{t}=\Delta_{11}^{-1}\Delta_{12}\mathrm{Sc}\left(\Delta, \Delta_{11}\right)^{-1}\mathrm{diag}(\mu_t)\bigtriangledown_{c_t}\bm{g}(x_t,c_t)\\&
			-\!\left(\Delta_{11}^{-1}+\!\Delta_{11}^{-1}\Delta_{12}\mathrm{Sc}\left(\Delta, \Delta_{11}\right)^{-1}\!\!\Delta_{21}\Delta_{11}^{-1} \right)\!\! \left(\!\bigtriangledown_{x_{t} c_t}l(x_t,c_t) +\!\!\sum_{m=1}^M\left(  \lambda_{m,t} +\mu_{m,t}\right)\bigtriangledown_{x_t c_t}g_m\left(x_{t},\bm{c}_{t}\! \right) \right) .
		\end{split}
	\end{equation}
	And the corresponding bound is
	\begin{equation}\label{eqn:gradopt_c_bound}
		\begin{split}
			&\left\| \bigtriangledown_{c_t}\bm{x}_{t}\right\|_2 \\
			\leq &\left\|\Delta_{11}^{-1}+\Delta_{11}^{-1}\Delta_{12}\mathrm{Sc}\left(\Delta, \Delta_{11}\right)^{-1}\Delta_{21}\Delta_{11}^{-1} \right\|_2 \left(\bar{\rho}_{l,x,c}+M(\bar{\lambda}+\bar{\mu})\bar{\rho}_{g,x,c} \right)\\
			&+\left\| \Delta_{11}^{-1}\Delta_{12}\mathrm{Sc}\left(\Delta, \Delta_{11}\right)^{-1}\mathrm{diag}(\mu_t)\bigtriangledown_{c_t}\bm{g}(x_t,c_t)\right\|_2 \\
			\leq& \frac{\bar{\rho}_{l,x,c}+M(\bar{\lambda}+\bar{\mu})\bar{\rho}_{g,x,c}}{\breve{\rho}_{l,g}}+\frac{Mq^2_{g,x}\bar{\mu}\bar{\rho}_{l,x,c}+M^2q^2_{g,x}(\bar{\mu}\bar{\lambda}+\bar{\mu}^2)\bar{\rho}_{g,x,c}}{\breve{\rho}^2_{l,g}\min\left\lbrace d(\bm{g},\bm{b}), \sigma_{p,\lambda}\right\rbrace}+\frac{M\Gamma_{g,x}\bar{\mu}\Gamma_{g,c}}{\breve{\rho}_{l,g}\min\left\lbrace d(\bm{g},\bm{b}), \sigma_{p,\lambda}\right\rbrace},
		\end{split}
	\end{equation}
	where the first inequality holds by triangle inequality and the assumptions that $\left\| \bigtriangledown_{x_{t} c_t}l(x_t,c_t)\right\|\leq \bar{\rho}_{l,x,c} $, $\left\| \bigtriangledown_{x_t c_t}g_m\left(x_{t},\bm{c}_{t} \right) \right\|\leq \bar{\rho}_{g,x,c} $ and $\lambda_{m,t}\leq \bar{\lambda}, \mu_{m,t}\leq \bar{\mu}$.
	
	Therefore, we prove that the Lipschitz constants of the optimization layer is bounded and we denote the upper bounds of $\left\| \bigtriangledown_{\lambda_t}\!x_{t}\right\|_2$, $\left\| \bigtriangledown_{\bm{b}_t}\!x_{t}\right\|_2$, and $\left\| \bigtriangledown_{\bm{c}_t}\!x_{t}\right\|_2$ as $\Gamma_{p,\lambda}, \Gamma_{p,b}$ and $\Gamma_{p,c}$, respectively.  
	
	Next we prove that the sensitivity of the total loss are bounded.
	Let $x_{j,t},c_{j,t},\lambda_{j,t},\bm{b}_{j,t}$ be the values
	flowing through our recurrent architectures for the sample $(\bm{c}_j, \bm{B}_j)$  while let $x'_{j,t}, c'_{j,t},\lambda'_{j,t},\bm{b}'_{j,t}$ be the values for $(\bm{c}'_j, \bm{B}'_j)$. Let $\omega_c=\max_{\bm{c},\bm{c}'\in \mathbb{C}}\left\|\bm{c}-\bm{c}' \right\| $ is the size of the parameter space $\mathbb{C}$, $\omega_b=\max_{\bm{B},\bm{B}'\in \mathbb{B}}\left\|\bm{B}-\bm{B}' \right\| $ is the size of the resource capacity space $\mathbb{B}$. By Lipschitz continuity of the optimization layer, we have
	\begin{equation}\label{eqn:loss_sensitive}
		\begin{split}
			&L(\pi_\theta,(\bm{c},\bm{B}))-L(\pi_\theta, (\bm{c}',\bm{B}'))\\
			=&\sum_{t=1}^Nl(x_{j,t},c_{j,t})-l(x'_{j,t},c'_{j,t})\\
			\leq& N\Gamma_{l,c}\omega_c+\Gamma_{l,x}\sum_{t=1}^N\left\|p(c_{j,t},\lambda_{j,t},\bm{b}_{j,t})-p(c'_{j,t},\lambda'_{j,t},\bm{b}'_{j,t}) \right\|  \\
			\leq &N\Gamma_{l,c}\omega_c+\Gamma_{l,x}\sum_{t=1}^N\left\|p(c_{j,t},\lambda_{j,t},\bm{b}_{j,t})-p(c'_{j,t},\lambda_{j,t},\bm{b}_{j,t}) \right\|\\
			&+\left\|p(c'_{j,t},\lambda_{j,t},\bm{b}_{j,t})-p(c'_{j,t},\lambda'_{j,t},\bm{b}_{j,t}) \right\|+\left\|p(c'_{j,t},\lambda'_{j,t},\bm{b}_t)-p(c'_{j,t},\lambda'_{j,t},\bm{b}'_{j,t}) \right\| \\
			\leq &N\Gamma_{l,c}\omega_c+N\Gamma_{l,x}\Gamma_{p,c}\omega_c+\Gamma_{l,x}\Gamma_{p,\lambda}\sum_{t=1}^N\left\|\lambda_{j,t}-\lambda'_{j,t} \right\| +\Gamma_{l,x}\Gamma_{p,b}\sum_{t=1}^N\left\|\bm{b}_{j,t}-\bm{b}'_{j,t} \right\|
		\end{split}
	\end{equation}
	If the Lipschitz constant of  the neural network is $\Gamma_f$, we have
	\begin{equation}\label{eqn:lambda_inequality_recur}
		\left\|\lambda_{j,t}-\lambda'_{j,t} \right\|\leq\Gamma_{f,b} \left\|\bm{b}_t-\bm{b}'_t \right\| +\Gamma_{f,c}\omega_c.
	\end{equation}
	By the expression of the remaining budget in Line \ref{alg:updating} of Algorithm \ref{alg:online_inference} and Lipchitz continuity of each layer, we have
	\begin{equation}\label{eqn:budget_inequaltiy_recur}
		\begin{split}
			\left\|\bm{b}_{j,t}-\bm{b}'_{j,t} \right\|&\leq \left\|\bm{b}_{j,{t-1}}-\bm{b}'_{j,{t-1}} \right\| +\left\| \bm{g}(x_{t-1},c_{t-1})-\bm{g}(x'_{t-1},c'_{t-1})\right\| \\
			&\leq \left\|\bm{b}_{j,{t-1}}-\bm{b}'_{j,{t-1}} \right\| +\Gamma_{g,x}\left\|x_{t-1}-x'_{t-1} \right\| +\Gamma_{g,c}\omega_c\\
			&\leq \left(1+\Gamma_{g,x}\Gamma_{p,b} \right) \left\|\bm{b}_{j,{t-1}}-\bm{b}'_{j,{t-1}} \right\|+\Gamma_{g,x}\Gamma_{p,\lambda}\left\|\lambda_{t-1}-\lambda'_{t-1} \right\|+\left(1+\Gamma_{p,c} \right) \omega_c\\
			&\leq \left(1+\Gamma_{g,x}\Gamma_{p,b}+\Gamma_{g,x}\Gamma_{p,\lambda}\Gamma_{f,b} \right) \left\|\bm{b}_{j,{t-1}}-\bm{b}'_{j,{t-1}} \right\|+\left(1+\Gamma_{p,c}+\Gamma_{g,x}\Gamma_{p,\lambda}\Gamma_{f,c} \right) \omega_c.
		\end{split}
	\end{equation}
	By performing induction based on recurrent inequality \eqref{eqn:budget_inequaltiy_recur}, we have
	\begin{equation}\label{eqn:budget_inequaltiy}
		\begin{split}
			&\left\|\bm{b}_{j,t}-\bm{b}'_{j,t} \right\|
			\leq \left(1+\Gamma_{g,x}\Gamma_{p,b}+\Gamma_{g,x}\Gamma_{p,\lambda}\Gamma_{f,b} \right)^{t-1} \omega_b +\\
			&\qquad \frac{\left(1+\Gamma_{p,c}+\Gamma_{g,x}\Gamma_{p,\lambda}\Gamma_{f,c} \right) \left(\left(1+\Gamma_{g,x}\Gamma_{p,b}+\Gamma_{g,x}\Gamma_{p,\lambda}\Gamma_{f,b} \right)^{t-1}-1 \right)}{\Gamma_{g,x}\Gamma_{p,b}+\Gamma_{g,x}\Gamma_{p,\lambda}\Gamma_{f,b}} \omega_c.
		\end{split}
	\end{equation}
	Combining with \eqref{eqn:lambda_inequality_recur}, we have
	\begin{equation}\label{eqn:lambda_inequality}
		\begin{split}
			&\left\|\lambda_{j,t}-\lambda'_{j,t} \right\|\leq\Gamma_{f,b}\left(1+\Gamma_{g,x}\Gamma_{p,b}+\Gamma_{g,x}\Gamma_{p,\lambda}\Gamma_{f,b} \right)^{t-1} \omega_b  \\
			&\qquad+\left( \Gamma_{p,b}\frac{\left(1+\Gamma_{p,c}+\Gamma_{g,x}\Gamma_{p,\lambda}\Gamma_{f,c} \right) \left(\left(1+\Gamma_{g,x}\Gamma_{p,b}+\Gamma_{g,x}\Gamma_{p,\lambda}\Gamma_{f,b} \right)^{t-1}-1 \right)}{\Gamma_{g,x}\Gamma_{p,b}+\Gamma_{g,x}\Gamma_{p,\lambda}\Gamma_{f,b}}+\Gamma_{f,c}\right) \omega_c.
		\end{split}
	\end{equation}
	Therefore, the loss difference in Eqn.\eqref{eqn:loss_sensitive} can be bounded as
	\begin{equation}\label{eqn:loss_sensitive_bound}
		\begin{split}
			&L(\pi_\theta,(\bm{c},\bm{B}))-L(\pi_\theta, (\bm{c}',\bm{B}'))\\
			\leq &N\Gamma_{l,c}\omega_c+N\Gamma_{l,x}\Gamma_{p,c}\omega_c+\Gamma_{l,x}\Gamma_{p,\lambda}\sum_{t=1}^N\left\|\lambda_{j,t}-\lambda'_{j,t} \right\| +\Gamma_{l,x}\Gamma_{p,b}\sum_{t=1}^N\left\|\bm{b}_{j,t}-\bm{b}'_{j,t} \right\| \\
			\leq&\omega_c \left( \frac{\Gamma_{l,x}(\Gamma_{p,\lambda}+1)N\Gamma_{p,b}\left(1+\Gamma_{p,c}+\Gamma_{g,x}\Gamma_{p,\lambda}\Gamma_{f,c} \right) \left(\left(1+\Gamma_{g,x}\Gamma_{p,b}+\Gamma_{g,x}\Gamma_{p,\lambda}\Gamma_{f,b} \right)^{t-1}-1 \right)}{\Gamma_{g,x}\Gamma_{p,b}+\Gamma_{g,x}\Gamma_{p,\lambda}\Gamma_{f,b}}\right.\\
			&\left.+N\Gamma_{l,c}+N\Gamma_{l,x}\Gamma_{p,c}+\Gamma_{l,x}\Gamma_{p,\lambda}N\Gamma_{f,c}\right)\\
			&+\omega_b\Gamma_{l,x}\left(N\Gamma_{l,x}\Gamma_{p,\lambda}\Gamma_{f,b}+N\Gamma_{p,b} \right)\left(1+\Gamma_{g,x}\Gamma_{p,b}+\Gamma_{g,x}\Gamma_{p,\lambda}\Gamma_{f,b} \right)^{t-1} \\
			\leq &\Gamma_{L,c}\omega_c+\Gamma_{L,b}\omega_b.
		\end{split}
	\end{equation}
\end{proof}

\textbf{Proof of Theorem \ref{thm:generalization}}
\begin{proof}
	If Lemma~\ref{lma:Lipschitzofloss} holds, we can use the generalization theorem for Rademacher complexity \cite{Rademacher_risk_bounds_bartlett2002rademacher,generalization_bounds_robust_attias2019improved} to bound the expected loss, which states that  for any $\theta\in \Theta$, with probability at least $1-\delta, \delta\in(0,1)$,
	\begin{equation}\label{eqn:generalization_rad}
		\begin{split}
			&\mathrm{E}\left[ L(\bm{h}_{\theta},\bm{c})\right] \\
			\leq&L\left(\bm{h}_{\theta},S\right) +2R_n(L\circ \mathbb{H})+\left( \Gamma_{L,c}\omega_c\!+\!\Gamma_{L,b}\omega_b\right) \!\sqrt{\frac{\ln(1/\delta)}{n}}.
		\end{split}
	\end{equation}
	For the learned weight $\hat{\theta}$, we have with probability at least $1-\delta, \delta\in(0,1)$,
	\begin{equation}\label{eqn:boundempiricalloss}
		\begin{split}
			L\left(\bm{h}_{\hat{\theta}},S\right)&=L\left(\bm{h}_{\hat{\theta}},S\right)-L\left(\bm{h}_{\hat{\theta}^*},S\right)+L\left(\bm{h}_{\hat{\theta}^*},S\right)\\
			&\leq \mathcal{E}\left(\bm{h}_{\hat{\theta}},S\right)  +L\left(\bm{h}_{\theta^*},S\right)\\
			&\leq \mathcal{E}\left(\bm{h}_{\hat{\theta}},S\right) +\mathrm{E}\left[ L(\bm{h}_{\theta^*})\right] +2R_n(L \circ\mathbb{H})\\&+\left( \Gamma_{L,c}\omega_c+\Gamma_{L,b}\omega_b\right) \sqrt{\frac{\ln(1/\delta)}{n}},
		\end{split}
	\end{equation}
	where the first inequality holds because $\hat{\theta}^*$ minimizes the empirical loss $L\left(\bm{h}_{\hat{\theta}^*},S\right)$, and the second inequality holds by applying the generalization bound \eqref{eqn:generalization_rad} to $L\left(\bm{h}_{\hat{\theta}^*},S\right)$. Combining Eqn.~\eqref{eqn:generalization_rad} and Eqn.~\eqref{eqn:boundempiricalloss} and applying union bound, we get
	the bound of expected loss 
	in
	Theorem~\ref{thm:generalization}.
\end{proof}

\subsection{Proof of Proposition~\ref{lma:rad_linear}}\label{sec:concreteML}
\begin{definition}[Rademacher Complexity]\label{def:rademacher}
	Let $\mathbb{H}=\left\lbrace \bm{h}: \mathbb{B}\times \mathbb{C}\rightarrow \mathbb{X}\right\rbrace $, where $\mathbb{B}, \mathbb{C}$, and $\mathbb{X}$ are the spaces of $\bm{B},\bm{c}$,  and $\bm{x}=[x_1,\cdots, x_N]^\top$ respectively, be the function space of online optimizer and denote the space of the total loss as $L\circ \mathbb{H}=\left\lbrace L(\bm{h}), \bm{h}\in\mathbb{H}\right\rbrace $. Given the dataset $S=\left\lbrace \left(\bm{c}_1,\bm{B}_1 \right),\cdots,  \left(\bm{c}_n,\bm{B}_n\right) \right\rbrace $, the Rademacher complexity regarding the space of total loss is
	\[
	R_n(L\circ\mathbb{H})=\mathrm{E}_{S}\mathrm{E}_{\nu}\left[\sup_{\bm{h}\in\mathbb{H}}\left(\frac{1}{n}\sum_{i=1}^n\nu_iL\left( \bm{h}(\bm{B}_i,\bm{c}_i)\right) \right) \right],
	\]
	where $\nu_1,\cdots, \nu_n$ are independently drawn from Rademacher distribution.
\end{definition}
\textbf{Proof of Proposition~\ref{lma:rad_linear}}
\begin{proof}
	We consider two concrete examples for the ML model --- linear
	and neural network models --- used in our unrolled
	architecture shown in Fig.~\ref{fig:illustration}. We start with a linear model to substantiate the ML model in \ouralg. Specifically, we have
	\begin{equation}\label{eqn:linearmodel}
		f_{\theta}(\bm{v})=\theta^\top\phi(\bm{v})
	\end{equation}
	where $\theta\in\Theta=\left\lbrace \theta\in\mathbb{R}^{M\times q}\mid, \left\| \theta\right\|_2\leq Z \right\rbrace $ is the weight to be learned, $\bm{v}=[\bm{b}_t^\top,c_t,\bar{t}]^\top$, $\phi: \mathbb{B}\times \mathbb{C}\times \mathbb{R}^{1}\rightarrow \mathbb{R}^q$ is a determinant feature mapping function which can be a manually-designed kernel or a pre-trained neural network.
	
	According to our architecture in Fig.~\ref{fig:illustration}, the final online optimizer $\bm{h}$ is a composition of the optimizer $p$ and the ML model, i.e. $\bm{h}=p\circ \bm{f}_{\theta}$ where $\bm{f}_{\theta}=[f_{\theta}(\bm{b}_1,c_1),\cdots, f_{\theta}(\bm{b}_t,c_
	t)]$ and $p$ is applied to each entry of  $\bm{f}_{\theta}$. By \eqref{eqn:loss_sensitive}, \eqref{eqn:lambda_inequality_recur}, and \eqref{eqn:budget_inequaltiy_recur}, for two different outputs $\lambda_1$ and $\lambda_1'$ of the neural network for $t=1$ and the corresponding action $x_t$, $x_t'$, we have
	\begin{equation}
		l(x_t,c_t)-l(x_t',c_t)\leq \Upsilon_t\|\lambda_1-\lambda_1'\|,
	\end{equation}
	where $\Upsilon_t=\Gamma_{l,x}\left(\Gamma_{p,\lambda}\Gamma_{f,b}+\Gamma_{l,x}\Gamma_{p,b}\right)\left(1+\Gamma_{g,x}\Gamma_{p,b}+\Gamma_{g,x}\Gamma_{p,\lambda}\Gamma_{f,b}\right)^{t-2}\Gamma_{g,x}\Gamma_{p,\lambda}$.
	According to the contraction lemma of Rademacher complexity in
	\cite{vector_radmacher_maurer2016vector}, 
	we have
	\begin{equation}\label{eqn:rad_loss_decom}
		R_n(L\circ \mathbb{H})\leq N \sum_{t=1}^N \Upsilon_t\sum_{m=1}^MR_n(\mathbb{F}_{1,m}),
	\end{equation}
	where $\mathbb{F}_{1,m}=\left\lbrace \bm{f}_{\theta,1,m}: \mathbb{B}\times \mathbb{C}\rightarrow \Lambda^1 \right\rbrace $ is the function domain regarding the $m$th dimension of $\lambda_t$ output by the linear ML model.
	
	Now we derive the Rademacher complexity of the ML model spaces $\mathbb{F}_{1,m}$. Let $\bm{v}_{i,t}=[\bm{b}_{t,i}^\top,c_{t,i},\bar{t}]^\top, i=1,\cdots, n$. Assume that $\|\theta_m\|_2= Z_m$ According to the definition of empirical Radmacher complexity, we have
	\begin{equation}\label{eqn:rad_f_entry}
		\begin{split}
			\hat{R}_n(\mathbb{F}_{1,m})&=\mathrm{E}_{\nu}\left[ \sup_{\bm{f}_{\theta,t,m}\in\mathbb{F}_{1,m}}\left( \frac{1}{n}\sum_{i=1}^n\nu_i\bm{f}_{\theta,t,m}(\bm{B}_i,\bm{c}_i)\right) \right] \\
			&=\mathrm{E}_{\nu}\left[ \sup_{\bm{f}_{\theta,t,m}\in\mathbb{F}_{1,m}}\left( \frac{1}{n}\sum_{i=1}^n\nu_i\theta_{m}^\top\phi(\bm{v}_{i,t})\right) \right] \\
			&=\frac{1}{n}\mathrm{E}_{\nu}\left[  Z_m\left\|\sum_{i=1}^n\nu_i\phi(\bm{v}_{i,t})\right\|_2\right] \\
			&\leq \frac{Z_m}{n} \sqrt{\mathrm{E}_{\nu}\left\|\sum_{i=1}^n\nu_i\phi(\bm{v}_{i,t})\right\|^2_2} \\
			&=\frac{Z_m}{n} \sqrt{\sum_{i=1}^n\left\|\phi(\bm{v}_{i,t})\right\|^2_2} \leq \frac{Z_mW}{\sqrt{n}},
		\end{split}
	\end{equation}
	where the third equality holds by Cauchy-Schwartz inequality, the first inequality holds by using Jensen's inequality for the expectation, the forth equality holds since $\mathrm{E}\left[\nu_i\nu_j \right]=0 $ when $i\neq j$ and the last inequality by the assumption that the norm of the feature mapping is at most $W$. Taking expectation over the dataset $S$, we have
	\begin{equation}\label{eqn:rad_f_entry_expected}
		R_n(\mathbb{F}_{1,m})=\mathrm{E}_S\left[ \hat{R}_n(\mathbb{F}_{1,m})\right] \leq \frac{Z_mW}{\sqrt{n}}
	\end{equation}
	
	Substituting inequality \eqref{eqn:rad_f_entry_expected} into \eqref{eqn:rad_loss_decom}, since $\sum_{m=1}^M\|\theta_m\|_2\leq \sqrt{M}\|\theta\|, $we get the conclusion
	\begin{equation}
		R_n(L\circ \mathbb{H})\leq \sqrt{M}N^2\Upsilon_N\frac{ZW}{\sqrt{n}}
	\end{equation}

	Next, We consider a fully neural network with $K$ hidden layers, each with the number of neurons no larger than $u$, which can be represented as
	\begin{equation}\label{eqn:neural_net}
		f_{\theta}(\bm{v})=\alpha_K\left( \theta_K\alpha_{K-1}\left(\theta_{K-1}\cdots\alpha_1\left(\theta_1\bm{v} \right)  \right) \right) ,
	\end{equation}
	where $\alpha_k(\cdot)$, for $k=1,\cdots, K$, is the activation function and $\theta_k$,
	for $k=1,\cdots, K$ is the weight for $k$-th layer.
	
	Same as Eqn.\eqref{eqn:rad_loss_decom}, we have
	\begin{equation}\label{eqn:rad_loss_decom_nn}
		R_n(L\circ \mathbb{H})\leq N \sum_{t=1}^N \Upsilon_t\sum_{m=1}^MR_n(\mathbb{F}_{1,m}),
	\end{equation}
	where now $\mathbb{F}_{1,m}=\left\lbrace \bm{f}_{\theta,t,m}: \mathbb{B}\times \mathbb{C}\rightarrow \Lambda^1 \right\rbrace $ is the function domain regarding the $m$-th dimension of $\lambda_t$ output by the neural network.
	
	Denote $\bm{v}_{i,t}=[\bm{b}_{t,i}^\top,c_{t,i},\bar{t}]^\top, i=1,\cdots, n$ and $\bm{V}=[\bm{v}_{1,t};\bm{v}_{2,t},\cdots, \bm{v}_{n,t}]$. Denote the convering number of the space $\mathbb{F}_{1,m}$ as $\mathcal{N}(\mathbb{F}_{1,m},\epsilon,2)$ with respect to Scale $\epsilon$ and 2-norm which is defined in 
	\cite{spectral_norm_rademacher_bartlett2017spectrally}
	The Lipschitz constants of all the activation functions $\alpha_k$ are less than or equal to $\Gamma_{\alpha}$. Assume that the spectral norm of the weight in each layer satisfy $\left\|\theta_k\right\|_2\leq Z_k, k=1,\cdots, K$. By Theorem 3.3 in 
	\cite{spectral_norm_rademacher_bartlett2017spectrally},
	the logarithm of the covering number $\mathcal{N}(\mathbb{F}_{1,m})$ of the space $\mathcal{N}(\mathbb{F}_{1,m},\epsilon,2)$ can be bounded as
	\begin{equation}
		\ln\mathcal{N}(\mathbb{F}_{1,m},\epsilon,2)\leq \frac{\|\bm{V}\|_F^2}{\epsilon^2}\left(\prod_{k=1}^{K}Z_k^2\Gamma_{\sigma}^2 \right)\left( \sum_{k=1}^K\left(\frac{\left\| \left( \theta_k-\bar{\theta}_k\right)^\top \right\|^{2/3}_{2,1} }{Z_k^{2/3}} \right) \right) ^3,
	\end{equation}
	where $\bar{\theta}_k,k=1,\cdots, K$ are reference matrices with the same dimensions as $\theta_k$. We choose reference matrices with the same spectral norm bound as $\theta_k$, and simplify the logarithm of the covering number as
	\begin{equation}\label{eqn:covering_number}
		\ln\mathcal{N}(\mathbb{F}_{1,m},\epsilon,2)\leq \frac{4K^3u^2\|\bm{V}\|_F^2}{\epsilon^2}\left(\prod_{k=1}^{K}Z_k^2\Gamma_{\sigma}^2 \right),
	\end{equation}
	where the inequality holds because $\left\| \left( \theta_k-\bar{\theta}_k\right)^\top \right\|_{2,1}\leq \left\| \theta^\top_k\right\|_{2,1}+\left\|\bar{\theta}^\top _k\right\|_{2,1}\leq \sqrt{u} \left( \left\| \theta^\top_k\right\|_{F}+\left\|\bar{\theta}^\top _k\right\|_{F}\right) \leq u\left( \left\| \theta^\top_k\right\|_{2}+\left\|\bar{\theta}^\top _k\right\|_{2}\right) \leq 2uZ_k$.
	
	Substituting \eqref{eqn:covering_number} into Dudley entropy Integral, we have
	\begin{equation}
		\begin{split}
			\hat{R}_n(\mathbb{F}_{1,m})&\leq \inf_{\gamma>0}\left(\frac{4\gamma}{\sqrt{n}}+\frac{12}{n}\int_{\gamma}^{\sqrt{n}}\sqrt{\ln\mathcal{N}(\mathbb{F}_{1,m},\epsilon,2)}\mathrm{d}\epsilon \right) \\
			&\leq \frac{24K^{3/2}u\|\bm{V}\|_F\Gamma_{\alpha}\prod_{k=1}^{K}Z_k}{n}\left(1+\ln\frac{n}{72K^{3/2}u\|\bm{V}\|^{1/2}\Gamma_{\alpha}\prod_{k=1}^{K}Z_k} \right) \\
			&\leq \tilde{O}\left(\frac{K^{3/2}u\Gamma_{\alpha}(\beta_b+\beta_c)\prod_{k=1}^{K}Z_k}{\sqrt{n}} \right) ,
		\end{split}
	\end{equation}
	where the last inequality holds because $\|\bm{V}\|_F\leq \sqrt{n}\left(\beta_b+\beta_c \right) $ where $\beta_b, \beta_c$ are the largest norm of budget $\bm{B}$ and parameter $c$, respectively.Taking expectation over the dataset $S$, we have
	\begin{equation}\label{eqn:rad_elementnn_expected}
		R_n(\mathbb{F}_{1,m})=\mathrm{E}_S\left[ \hat{R}_n(\mathbb{F}_{1,m})\right] \leq \tilde{O}\left(\frac{K^{3/2}u\Gamma_{\alpha}(\beta_b+\beta_c)\prod_{k=1}^{K}Z_k}{\sqrt{n}} \right).
	\end{equation}
	
	Substituting inequality \eqref{eqn:rad_elementnn_expected} into \eqref{eqn:rad_loss_decom_nn}, we get the conclusion
	\begin{equation}
		R_n(L\circ \mathbb{H})\leq \sqrt{M}N^2\Upsilon_N\tilde{O}\left(\frac{K^{3/2}u\Gamma_{\alpha}(\beta_b+\beta_c)\prod_{k=1}^{K}Z_k}{\sqrt{n}} \right).
	\end{equation}
\end{proof}

\subsection{Proof of Proposition~\ref{thm:generalization_sgd}}

\begin{proof}
By Lemma 4.4 in \cite{optimization_learning_bottou2018optimization},
with the assumptions about training objective and instance sequence held, we have
\begin{equation}
\begin{split}
    &\mathrm{E}\left[L(h_{\hat{\theta}_{i+1}})-L(h_{\hat{\theta}_{i}})\right]\\
    \leq&  -\left(\iota-\frac{1}{2}\bar{\alpha}\Gamma_{\triangledown L, \theta}\left(\varpi_V+\iota_G^2\right)\right)\bar{\alpha}\left[\left\|\bigtriangledown_{\hat{\theta}}\mathrm{E}[L(h_{\hat{\theta}_{i}})]\right\|_2^2\right]+\frac{1}{2}\bar{\alpha}^2\Gamma_{\triangledown L, \theta}\varpi\\
    \leq & -\iota\bar{\alpha}\varsigma\mathrm{E}\left[ L(h_{\hat{\theta}_{i+1}})-L(h_{\theta^{*}})\right]+\frac{1}{2}\bar{\alpha}^2\Gamma_{\triangledown L, \theta}\varpi,
\end{split}
\end{equation}
where the expectation is taken over the randomness of the parameter sequence up to round $i$, the last inequality holds by the choice of $\bar{\alpha}$ and the assumption of Polyak-Lojasiewicz inequality. Subtracting $\mathrm{E}\left[ L\left(h_{\theta^*}\right)\right]$ from both sides, we have
\begin{equation}
\begin{split}
    &\mathrm{E}\left[L(h_{\hat{\theta}_{i+1}})-L(h_{\theta^*})\right]
    \leq \left(1-\iota\bar{\alpha}\varsigma\right)\mathrm{E}\left[ L(h_{\hat{\theta}_{i+1}})-L(h_{\theta^{*}})\right]+\frac{1}{2}\bar{\alpha}^2\Gamma_{\triangledown L, \theta}\varpi,
    \end{split}
\end{equation}
and so 
\begin{equation}\label{eqn:induction}
\begin{split}
    &\mathrm{E}\left[L(h_{\hat{\theta}_{i}})-L(h_{\theta^*})\right]-\frac{\bar{\alpha}\Gamma_{\triangledown L, \theta}\varpi}{2\iota\varsigma}\\
    \leq& \left(1-\iota\bar{\alpha}\varsigma\right)\mathrm{E}\left[ L(h_{\hat{\theta}_{i-1}})-L(h_{\theta^{*}})\right]+\frac{1}{2}\bar{\alpha}^2\Gamma_{\triangledown L, \theta}\varpi-\frac{\bar{\alpha}\Gamma_{\triangledown L, \theta}\varpi}{2\iota\varsigma}\\
    =&\left(1-\iota\bar{\alpha}\varsigma\right)\left(\mathrm{E}\left[L(h_{\hat{\theta}_{i-1}})-L(h_{\theta^*})\right]-\frac{\bar{\alpha}\Gamma_{\triangledown L, \theta}\varpi}{2\iota\varsigma}\right).
    \end{split}
\end{equation}
Thus, using inequality \eqref{eqn:induction} recurrently, we have
\begin{equation}
\begin{split}
    \mathrm{E}\left[L(h_{\hat{\theta}_{i}})-L(h_{\theta^*})\right]-\frac{\bar{\alpha}\Gamma_{\triangledown L, \theta}\varpi}{2\iota\varsigma}
    \leq \left(1-\iota\bar{\alpha}\varsigma\right)^i\left(\mathrm{E}\left[L(h_{\hat{\theta}_{0}})-L(h_{\theta^*})\right]-\frac{\bar{\alpha}\Gamma_{\triangledown L, \theta}\varpi}{2\iota\varsigma}\right).
    \end{split}
\end{equation}
Thus, we can bound the expected cost gap at round $i$ as 
\begin{equation}
\begin{split}
   \mathrm{E}\left[L(h_{\hat{\theta}_{i}})-L(h_{\theta^*})\right]\leq \frac{\bar{\alpha}\Gamma_{\triangledown L, \theta}\varpi}{2\iota\varsigma}
    +O\left(\left(1-\iota\bar{\alpha}\varsigma\right)^i\right),
    \end{split}
\end{equation}
where the expectation is taken over the parameter sequence at and before round $i$. 
\end{proof}

\end{document}